\documentclass[accepted]{uai2025} 
                        
\usepackage[american]{babel}
\usepackage{csquotes}

\usepackage{natbib} 
\usepackage{mathtools} 
\usepackage{amssymb, amsthm}
\usepackage{subcaption}
\usepackage{relsize}
\usepackage{multirow}
\usepackage[ruled,vlined,linesnumbered]{algorithm2e}
\usepackage{float}

\usepackage{booktabs} 
\usepackage{tikz} 

\newtheorem{theorem}{Theorem}[section]

\newtheorem{remark}{Remark}

\title{RCAP: Robust, Class-Aware, Probabilistic Dynamic Dataset Pruning}

\author[1]{\href{mailto:<atif.hit.hassan@kgpian.iitkgp.ac.in>?Subject=Your UAI 2025 paper}{Atif Hassan}{}}
\author[2]{Swanand Khare}
\author[1]{Jiaul H. Paik}
\affil[1]{%
    Department of Artificial Intelligence\\
    IIT Kharagpur\\
    Kharagpur, West Bengal, India
}
\affil[2]{%
    Department of Mathematics\\
    IIT Kharagpur\\
    Kharagpur, West Bengal, India
}
  
\begin{document}
\maketitle

\begin{abstract}
Dynamic data pruning techniques aim to reduce computational cost while minimizing information loss by periodically selecting representative subsets of input data during model training. 
However, existing methods often struggle to maintain strong worst-group accuracy, particularly at high pruning rates, across balanced and imbalanced datasets. 
To address this challenge, we propose RCAP, a Robust, Class-Aware, Probabilistic dynamic dataset pruning algorithm for classification tasks. 
RCAP applies a closed-form solution to estimate the fraction of samples to be included in the training subset for each individual class. 
This fraction is adaptively adjusted in every epoch using class-wise aggregated loss.
Thereafter, it employs an adaptive sampling strategy that prioritizes samples having high loss for populating the class-wise subsets.
We evaluate RCAP on six diverse datasets ranging from class-balanced to highly imbalanced using five distinct models across three training paradigms: training from scratch, transfer learning, and fine-tuning. 
Our approach consistently outperforms state-of-the-art dataset pruning methods, achieving superior worst-group accuracy at all pruning rates. Remarkably, with only $10\%$ data, RCAP delivers $>1\%$ improvement in performance on class-imbalanced datasets compared to full data training while providing an average $8.69\times$ speedup. 
The code can be accessed at \url{https://github.com/atif-hassan/RCAP-dynamic-dataset-pruning}
\end{abstract}

\section{Introduction}\label{sec:intro}
The remarkable success of deep learning across domains such as computer vision \citep{DBLP:conf/cvpr/HeZRS16,DBLP:conf/iclr/DosovitskiyB0WZ21}, natural language processing \citep{DBLP:conf/nips/BrownMRSKDNSSAA20,radford2019language,DBLP:journals/corr/abs-2303-08774}, and speech \citep{DBLP:conf/icml/RadfordKXBMS23,DBLP:conf/nips/BaevskiZMA20} is largely fueled by training massive networks on datasets with millions or even billions of samples. 
However, this scale of training demands exorbitant computational resources over prolonged periods, incurring unsustainable monetary costs \citep{DBLP:conf/icml/MindermannBRS0X22}. 
These expenses not only limit accessibility for resource-constrained researchers but also discourage investment in model refinement activities like hyper-parameter tuning and architecture search. 
Consequently, reducing training costs has emerged as a critical research challenge in deep learning.

One promising approach to mitigate these costs is to reduce the number of training updates which can be achieved by shrinking the dataset size.
Approaches such as dataset distillation \citep{DBLP:conf/wacv/ZhaoB23,DBLP:conf/cvpr/Cazenavette00EZ22a}, coreset selection \citep{DBLP:conf/icml/XiaLZWWL24,DBLP:conf/icml/YangCGZLZN24,DBLP:conf/iclr/ZhengLL023} and data pruning \citep{DBLP:conf/cvpr/ZhangDLXZ24,DBLP:conf/icml/YangCGZLZN24,okanovic2024repeated,qin2024infobatch} have garnered attention with data pruning striking the best balance between performance and training cost by removing the least informative examples \citep{paul2021deep}. 

Pruning methods typically use scoring mechanisms to identify the most informative samples for training. 
Static pruning techniques \citep{paul2021deep,DBLP:conf/iclr/0006XP0S023,DBLP:conf/cvpr/ZhangDLXZ24,DBLP:conf/icml/YangCGZLZN24} select a fixed subset prior to training, discarding the remaining data to reduce storage and computation. 
However, their scoring mechanism relies on training a model for multiple epochs before determining sample importance. 
This process is not only expensive but also model-dependent, thus restricting its applicability to diverse downstream architectures.
Dynamic dataset pruning, in contrast, recomputes subsets during training, leveraging accessible metrics like per-sample loss to adaptively select data for each epoch \citep{raju2021accelerating,qin2024infobatch,okanovic2024repeated}. 
This dynamic approach ensures that the sampled subset evolves with model training, offering near loss-less average performance even at high pruning rates while reducing overall training time.

\subsection{Motivation}
Developing robust models is a crucial aspect of real-world AI applications as it mitigates bias against under-represented/minority groups.
However, existing state-of-the-art dynamic pruning algorithms, such as RS2 \citep{okanovic2024repeated} and InfoBatch \citep{qin2024infobatch}, overlook a critical metric: worst-group accuracy, essential for evaluating model robustness, especially in class-imbalanced datasets. 
Moreover, even in class-balanced datasets these methods often neglect class-specific hardness, achieving strong average performance but underperforming on harder or minority groups. 
For instance in CIFAR10, certain classes, such as cats and dogs, accumulate higher loss in comparison to other groups, leading to non-robust models with poor worst-class performance \citep{DBLP:journals/corr/abs-2404-05579}.
Thus, we aim to answer the following question, 
\begin{displayquote}
\textquote{\textit{Does incorporating class hardness, while performing data pruning, enhance model robustness across both balanced and imbalanced data settings?}}
\end{displayquote}

\subsection{Our Contribution}
We propose RCAP, a novel, Robust, Class-Aware, Probabilistic dynamic dataset pruning algorithm for classification tasks.
RCAP automatically determines the appropriate subset size for individual classes through a parameter which is updated in every epoch based on the aggregated class-wise loss of the previous epoch.
Thereafter, RCAP prioritizes samples with higher loss for each subset by sampling from a distribution over per-sample losses.

We evaluate RCAP across a diverse set of datasets, spanning various scales and class imbalance levels.
These include class-balanced datasets of medium scale (CIFAR10 and CIFAR100), a class-balanced large-scale dataset (ImageNet), a moderately imbalanced small-scale dataset (Waterbirds), a relatively high imbalance medium-scale dataset (CelebA), and an extremely imbalanced large-scale dataset (iNaturalist). 
Our experiments employ five distinct network architectures, ResNet18, ResNet50, EfficientNetV2, Dinov2 and EfficientFormerV2 across three training paradigms: training from scratch, transfer learning, and fine-tuning.

We compare against seven state-of-the-art baselines, including both dynamic and static data pruning techniques.
To the best of our knowledge, this is the first comprehensive evaluation of dynamic dataset pruning algorithms in both class-balanced and imbalanced data settings in terms of worst-group performance.
The results demonstrate that RCAP consistently surpasses all methods, achieving significantly superior worst-group accuracy, especially at high pruning rates across all architectures, datasets and training paradigms.

\section{Preliminaries}
\subsection{Notations}
\noindent We denote $\mathcal{S}=\{(X_i,y_i)\}_{i=1}^{n}$ as a labelled set of input and target pairs. 
Here, $X_i\in\mathcal{X}$ and $y_i\in\mathbb{N}_c$ where $\mathcal{X}$ is the input space while $\mathbb{N}_c=\{1,\cdots, c\}$ with $c$ being the number of classes and $n$ the total number of samples. 
Here, $(\mathbf{X}, Y)\sim\mathcal{P}_\mathcal{D}$ where $\mathcal{P}_\mathcal{D}$ is the underlying distribution. 
Given a label $j\in\mathbb{N}_c$, define $\mathcal{S}_j=\{(X_k,y_k)\}_{k=1}^{n_j}$ where $\forall k, y_k=j$. 
Then clearly, $\mathcal{S}=\bigcup_{j=1}^{c}\mathcal{S}_j$ and $n=\sum_jn_j$. 
Let $r\in(0,1)$ be the pruning rate supplied by the user such that the total number of samples to be selected is $(1-r)n$. 
We define the retain set, $\mathcal{S}^t\subset \mathcal{S}$ as the subset of samples selected for training at epoch $t$ where $\lvert\mathcal{S}^t\rvert=(1-r)n$. 
Here, $t=\{1,2,\cdots,T\}$ where $T$ is the total number of epochs.
The retain set comprises class-wise subsets, $\mathcal{S}_j^t=\mathcal{S}^t\cap\mathcal{S}_j, ~\forall j\in\mathbb{N}_c$ where $\lvert\mathcal{S}_j^t\rvert = \alpha^t_jn_j$ such that, $\sum_{j=1}^c\alpha_j^{t}n_j=(1-r)n$. 
Here, $\alpha_j^{t}$ is the fraction of samples to be selected to form the subset for class $j$ at epoch $t$.
The set of unused samples, $\mathcal{S}\setminus \mathcal{S}^t$, at epoch $t$ form the pruned set.
Let $f_\theta(\cdot)$ be any arbitrary model parameterized by $\theta\in\mathbb{R}^m$. 
Let $\widetilde{f}_{\theta^t}(X_i)=\sigma\left(f_{\theta^t}(X_i)\right)\in\mathbb{R}^c$ at epoch $t$ such that $\forall t,\forall i,~\left\lVert \widetilde{f}_{\theta^t}(X_i)\right\rVert_1=1$. 
Here, $\sigma(\cdot)$ is the Softmax function.
The loss function is denoted as, $L:\mathbb{R}^c\times\mathbb{N}_c\rightarrow \mathbb{R}$ with its value at epoch $t$ for some input $X_i$ being represented as $L\left(\widetilde{f}_{\theta^t}\left(X_i\right),y_i\right)$. 
For brevity, we represent the loss at epoch $t$ for some input $X_i$ as $L\left(\widetilde{f}_{\theta^t}(X_i)\right)$.
The derivative of $L\left(\widetilde{f}_{\theta^t}\left(X_i\right)\right)$ at epoch $t$ for any input $X_i$ is denoted as $\nabla_{\theta^t} L\left(\widetilde{f}_{\theta^t}(X_i)\right)$. 
Let $\mathcal{B}^p$ denote a batch of examples at iteration $p$ with $\lvert\mathcal{B}^p\rvert=b$ being the batch size. 
Then the total number of iterations at epoch $t$ over the entire dataset and a subset, $\mathcal{S}^t\subset\mathcal{S}$, are $\left\lceil|\mathcal{S}|/b\right\rceil$ and $\left\lceil|\mathcal{S}^t|/b\right\rceil$, respectively. 
We use $\eta$ to denote the learning rate.

\section{RCAP}\label{sec:pddp}
An effective data pruning algorithm should account for class-wise performance when selecting samples for the retain set. 
This is because the performance of individual groups/classes vary for a given classification task, thus requiring non-uniform representation in the selected subset.
Some methods such as Data Diet \citep{paul2021deep}, implicitly address this by inducting high-error samples into the retain set. 
However, at high pruning rates, such strategies risk discarding classes with consistently low-error samples. 
Other techniques such as MetriQ \citep{DBLP:journals/corr/abs-2404-05579} incorporate class-wise performance but rely on ad-hoc rules to determine sample allocation. 
To overcome these limitations, we propose the following two fundamental problems that any effective data pruning algorithm should solve:
\begin{itemize}
    \item Determining the appropriate subset size for each class in the retain set.
    \item Selecting the most informative samples within each subset.
\end{itemize}
We address the first problem by adaptively adjusting the class-wise subset size in each epoch, as formalized in Theorem \ref{theorem:first_and_last} (Section \ref{sec:module1}). 
The second problem is tackled through a novel epoch-wise adaptive sampling strategy, detailed in Section \ref{sec:module2}.

\subsection{Adaptive per-Class Subset Size}\label{sec:module1}
Allocating more training samples to classes that a model perceives as difficult can lead to performance improvements on underrepresented or challenging groups \cite{DBLP:journals/corr/abs-2404-05579}. 
Theorem \ref{theorem:first_and_last} formalizes this intuition, demonstrating that classes with higher loss values should have a proportionally larger representation in the training subset. 

\begin{theorem}
    Let the given total empirical error be denoted by,
    \begin{equation}
        \begin{split}
            &E = \sum_j \frac{p_j}{\alpha_jn_j}\widetilde{E}_j\\
            \text{ where }~~&\widetilde{E}_j = \sum_{X_i\in\widetilde{\mathcal{S}}_j}L\left(\widetilde{f}_{\theta}\left(X_i\right)\right) ,~~\widetilde{\mathcal{S}}_j\subseteq\mathcal{S}_j~~\text{ and }~~p_j=\frac{n_j}{n}
        \end{split}
        \nonumber
    \end{equation}
    Then, under the assumption of full batch gradient descent, the optimal solution to the minimization problem
    \begin{equation}
        \underset{\alpha_j}{\min}~ E
        \nonumber
    \end{equation}
    \begin{equation}
    \begin{split}
        &~~~\textrm{ subject to }~~~
        \sum_{j=1}^c\alpha_jn_j=(1-r)n\\
        &~~~\text{ is given by }~~~
        \widehat{\alpha}_j = \frac{\sqrt{p_j\widetilde{E}_j}}{\underset{j}{\sum}\sqrt{p_j\widetilde{E}_j}}(1-r)\frac{n}{n_j}
    \end{split}
    \nonumber
    \end{equation}
    \label{theorem:first_and_last}
\end{theorem}

\begin{proof}
    Introducing the Lagrange multiplier $\lambda$, the optimization problem becomes,
    \begin{equation}
        G = E + \lambda\left(\sum_{j=1}^c\alpha_jn_j-(1-r)n\right)
        \label{eqn:full_eqn}
        \nonumber
    \end{equation}
    If $\left(\widehat{\alpha}_j, \widehat{\lambda}\right)$ is an optimal pair, then the optimality conditions imply:
    \begin{equation}
        \begin{split}
            {\frac{\partial G}{\partial {\alpha}_j}}\Bigg|_{\left(\widehat{\alpha}_j, \widehat{\lambda}\right)} &= -\frac{p_j\widetilde{E}_j}{{\widehat{\alpha}}_j^2n_j} + \widehat{\lambda} n_j = 0\\
            \implies \widehat{\alpha}_j &= \frac{\sqrt{p_j\widetilde{E}_j}}{\sqrt{\widehat{\lambda}}n_j}
        \end{split}
        \label{eqn:alpha_j}
    \end{equation}
    Substituting the value of $\widehat{\alpha}_j$ in the constraint gives us,
    \begin{equation}
        \begin{split}
            &\frac{1}{\sqrt{\widehat{\lambda}}}\sum_j\sqrt{p_j\widetilde{E}_j} = (1-r)n\\
            \implies &\frac{1}{\sqrt{\widehat{\lambda}}} = \frac{(1-r)n}{\sum_j\sqrt{p_j\widetilde{E}_j}}
        \end{split}
        \label{eqn:lambda}
    \end{equation}
    Replacing the value of $\sqrt{\widehat{\lambda}}$ from Eqn. \ref{eqn:lambda} in Eqn. \ref{eqn:alpha_j}, we get,
    \begin{equation}
        \widehat{\alpha}_j = \frac{\sqrt{p_j\widetilde{E}_j}}{\sum_j\sqrt{p_j\widetilde{E}_j}}(1-r)\frac{n}{n_j}
        \label{eqn:pre-final}
    \end{equation}
\end{proof}

\begin{remark}
    {\rm
    Equation~\ref{eqn:pre-final} provides a closed-form solution for the optimal class-wise retention fractions, indicating that classes with higher cumulative error should be allocated a larger share of the remaining data.
    In practice, when implementing Theorem~\ref{theorem:first_and_last}, we instantiate the quantities $\widetilde{E}_j$ using the class-wise errors computed at epoch $t$, i.e., $\widetilde{E}_j := \widetilde{E}_j^{\,t}$ while $\widetilde{\mathcal{S}}_j:=\mathcal{S}_j^t$, and we use the resulting closed-form expression to obtain $\widehat{\alpha}_j^{\, t+1}$ for the next epoch. 
    This history-based strategy, i.e, using losses or gradient norms from a previous epoch (or iteration) to guide the sampling distribution in the next epoch, is well established in adaptive data selection and curriculum learning methods \cite{DBLP:journals/corr/LoshchilovH15,DBLP:journals/corr/SchaulQAS15,DBLP:journals/corr/abs-1910-00762}. 
    Thus, applying Theorem 1 in the online setting, we substitute epoch-$t$ statistics as a proxy for the true (unknown) optimal statistics at epoch $t+1$, yielding,
    \begin{equation}
        \widehat{\alpha}_j^{t+1} = \frac{\sqrt{p_j\widetilde{E}_j^t}}{\sum_j\sqrt{p_j\widetilde{E}_j^t}}(1-r)\frac{n}{n_j}
    \label{eqn:final}
    \end{equation}
}
\end{remark}

\textbf{Implementation Detail:} 
Note that $\widetilde{E}^0_j$ is the class-wise aggregated loss at model initialization which determines $\widehat{\alpha}_j^1$.
Furthermore, closely inspecting Eqn. \ref{eqn:final} reveals that the condition $\widehat{\alpha}_j^{t+1}>1$ is plausible as the fraction is unconstrained. 
To mitigate this issue, we perform the following operation,
{\small
\begin{equation}
\begin{split}
    \widehat{\alpha}_j^{t+1} &= \begin{cases}
        $1$&\text{ if }~ \widehat{\alpha}_j^{t+1}>1\\
        \frac{\sqrt{p_j\widetilde{E}_j^{t}}}{\sum_j\left(\sqrt{p_j\widetilde{E}_j^{t}}\right)m_j}\frac{(1-r)n-k}{n_j}&\text{ otherwise}
    \end{cases}\\
    \text{where,}\\
    m_j &= \begin{cases}
        $1$~~&\text{ if }~ \widehat{\alpha}_j^{t+1}<1\\
        $0$~~&\text{ otherwise}
    \end{cases}\\
    k &= \sum_j \left(1-m_j\right)n_j
\end{split}
\label{eqn:reweighting_alpha}
\end{equation}
}
In doing so, we guarantee that $\widehat{\alpha}_j^{t+1}\leq 1$ with excess values being re-distributed among the remaining classes.

\subsection{Adaptive per-Class Sample Selection}\label{sec:module2}
The goal of any dynamic dataset pruning algorithm is to train a model on a carefully selected subset of data at each epoch such that the model's performance is indistinguishable from a model trained on the full dataset. 
Formally, this goal can be expressed as,
\begin{equation}
    \mathlarger{\mathop{\mathbb{E}}}_{(X_i,y_i)\sim \mathcal{P}_\mathcal{D}}\left[\left\lvert L\left(\widetilde{f}_{\widetilde{\theta}^T}\left(X_i\right)\right) - L\left(\widetilde{f}_{\theta^T}\left(X_i\right)\right)\right\rvert\right] \leq \epsilon\label{eqn:data_prune_obj}
\end{equation}
where,
{\small
\begin{align}
    \theta^T &= \theta^1 - \eta\sum_{t=1}^{T}\sum_{p=1}^{\left\lceil\frac{\lvert \mathcal{S}\rvert}{b}\right\rceil}\frac{1}{b}\sum_{\left(X_i,y_i\right)\in \mathcal{B}^p}\nabla_{\theta^{t,p}} L\left(\widetilde{f}_{\theta^{t,p}}\left(X_i\right)\right)\label{eqn:theta_recur}\\
    \widetilde{\theta}^T &= \theta^1 - \eta\sum_{t=1}^{T}\sum_{p=1}^{\left\lceil\frac{\lvert \mathcal{S}^t\rvert}{b}\right\rceil}\frac{1}{b}\sum_{\left(\widetilde{X}_i,\widetilde{y}_i\right)\in \widetilde{\mathcal{B}}^p}\nabla_{\widetilde{\theta}^{t,p}} L\left(\widetilde{f}_{\widetilde{\theta}^{t,p}}\left(\widetilde{X}_i\right)\right)\label{eqn:theta_cap_recur}
\end{align}
}

Here, $\mathcal{B}^p$ and $\widetilde{\mathcal{B}}^p$ are batches sampled from $\mathcal{S}$ and $\mathcal{S}^t$, respectively, at iteration $p$.
Here, $\theta^T$ and $\widetilde{\theta}^T$ are the parameters obtained after training on $\mathcal{S}$ and its subset, respectively.
Similarly, $\theta^{t,p}$ and $\widetilde{\theta}^{t,p}$ are the parameters obtained at epoch $t$ and iteration $p$ after training on $\mathcal{S}$ and $\mathcal{S}^t$, respectively.
We now look at the condition to achieve Eqn. \ref{eqn:data_prune_obj}. 
Let $\frac{1}{b}\sum_{\left(X_i,y_i\right)\in \mathcal{B}^p}\nabla_{\theta^{t,p}} L\left(\widetilde{f}_{\theta^{t,p}}(X_i)\right)=g^{t,p}$ and $\frac{1}{b}\sum_{\left(X_i,y_i\right)\in \widetilde{\mathcal{B}}^p}\nabla_{\widetilde{\theta}^{t,p}} L\left(\widetilde{f}_{\widetilde{\theta}^{t,p}}(X_i)\right)=\widetilde{g}^{t,p}$.
Assuming that $L$ is Lipschitz continuous having Lipschitz constant $K_1$ with respect to the change in parameters, we get:
\begin{equation}
    \left\lvert L\left(\widetilde{f}_{\widetilde{\theta}^T}(X_i)\right)-L\left(\widetilde{f}_{\theta^T}(X_i)\right)\right\rvert \leq K_1 \left\lVert\widetilde{\theta}^T - \theta^T\right\rVert_2
    \label{eqn:lipscitz}
\end{equation}
Replacing $\widetilde{\theta}^T$ and $\theta^T$ from Eqns. \ref{eqn:theta_recur} and \ref{eqn:theta_cap_recur} in Eqn. \ref{eqn:lipscitz}, and taking expectation on both sides, we get:
{\small
\begin{equation}
    \begin{split}
    &\mathlarger{\mathlarger{\mathop{\mathbb{E}}}}_{\left(X_i,y_i\right)\sim \mathcal{P}_\mathcal{D}}\left[\left\lvert L\left(\widetilde{f}_{\widetilde{\theta}^T}(X_i)\right)-L\left(\widetilde{f}_{\theta^T}(X_i)\right)\right\rvert\right] \leq\\ K_1\eta &\mathlarger{\mathlarger{\mathop{\mathbb{E}}}}_{\left(X_i,y_i\right)\sim \mathcal{P}_\mathcal{D}}\left[\left\lVert\sum_{t=1}^{T}\left(\sum_{p=1}^{\left\lceil\frac{\lvert \mathcal{S}\rvert}{b}\right\rceil}g^{t,p} - \sum_{p=1}^{\left\lceil\frac{\lvert \mathcal{S}^t\rvert}{b}\right\rceil}\widetilde{g}^{t,p}\right)\right\rVert_2\right]
    \end{split}
    \label{eqn:min_orig_data_prune}
\end{equation}
}

Hence, to achieve Eqn. \ref{eqn:data_prune_obj}, the right-hand-side in Eqn. \ref{eqn:min_orig_data_prune} needs to be minimized.
One can observe that in each epoch, the term $\sum_{p=1}^{\left\lceil\left(\lvert\mathcal{S}\rvert/b\right)\right\rceil}g^{t,p}$ is dominated by the samples with the largest gradient norm. 
We empirically find that the cross-entropy loss and the magnitude of the gradient exhibit a monotonic relation (see Section \ref{sec:simulation} in the Supplementary Materials).
This empirical relation is further reinforced by \cite{paul2021deep}, as they observe that “examples that are learned faster and maintain small error over training have a smaller GraNd score on average,” where the GraNd score is the gradient norm of a sample.
Thus, we choose to form $\mathcal{S}^t$ with high-loss samples to approximately minimize the right-hand side in Eqn. \ref{eqn:min_orig_data_prune}. 
A na\"ive approach involves sorting samples in $\mathcal{S}$ by their loss values which is computationally expensive $\left(O(\log n)\text{ per sample}\text{, e.g. \citep{paul2021deep}}\right)$.
Instead, RCAP samples from every $\mathcal{S}_j$ by defining $\mathcal{S}_j^t\subseteq \mathcal{S}_j$ as the set of examples sampled at epoch $t$ for class $j$ in the following manner.
{\small
    \begin{subequations}
        \begin{align}
            \mathcal{S}_j^{t+1} &= \{\left(X_i,y_i\right)\}_{i=1}^{\widehat{\alpha}_j^{t+1}n_j}\sim \mathcal{P}_j^{t+1}(X_i)\label{eqn:sampling1}\\
            \mathcal{P}_j^{t+1}(X_i) &= \frac{\mathlarger{\mathlarger{e}}^{\left(\phi_j^{t}(X_i)/\beta\right)}}{\sum_{X_q\in\mathcal{S}_j}\mathlarger{\mathlarger{e}}^{\left(\phi_j^{t}(X_q)/\beta\right)}}\label{eqn:sampling2}\\
            \phi_j^{t}(X_i)&=\begin{cases}
                \gamma\left(L\left(\widetilde{f}_{\theta^t}(X_i)\right), j\right) ~~&\text{ if } X_i\in \mathcal{S}_j^{t}\\
                \phi_j^{t-1}(X_i)~~&\text{ otherwise}
            \end{cases}\label{eqn:sampling3}\\
            \gamma(x, j) &= \min\left(x, \max\left(\phi_j^0(X_i)\right)\right)~~\forall X_i\in\mathcal{S}_j\label{eqn:clipping_function}\\
            \phi_j^0(X_i) &= L\left(\widetilde{f}_{\theta^0}(X_i)\right)\label{eqn:sampling4}
        \end{align}
        \label{eqn:sampling}
    \end{subequations}
}

Before training ensues, all examples are forward passed through a randomly initialized network and the corresponding loss values are stored in $\phi_j^0$ as shown in Eqn. \ref{eqn:sampling4}.
These loss values correspond to completely random predictions.
Next, $\phi_j^0$ is used to compute the aggregate class-wise losses, $\widetilde{E}_j^0$, that determine the fraction of samples to be allocated per class, $\widehat{\alpha}_j^1$, as shown in Eqn. \ref{eqn:final}. 
Next, a class-wise probability distribution is generated over the collected loss values, $\phi_j^0$, using a Softmax function with $\beta$ as the temperature hyper-parameter as shown in Eqn. \ref{eqn:sampling2}. 
The training subset is then generated by sampling over this distribution as per Eqn. \ref{eqn:sampling1}. 
The model is then trained using these samples. 
Following an epoch of training, $\phi_j^0$ is updated with the new loss values corresponding to the selected samples, forming $\phi_j^1$ as per Eqn. \ref{eqn:sampling3} which in turn determines $\widehat{\alpha}_j^2$. 
The distribution is updated using Eqn. \ref{eqn:sampling2} and sampling re-occurs as per Eqn. \ref{eqn:sampling1}. 
This iterative process continues until the end of training.
Fig. \ref{fig:RCAP_overview} gives a graphical overview of the sequence of steps involved in each epoch. 
It is important to note that the softmax-based sampling distribution built from the loss values can become highly skewed due to the presence of a few large values, leading to unstable or biased subset selection. 
Hence, we define a clipping function in Eqn. \ref{eqn:clipping_function} with the clipping threshold as the maximum loss observed in epoch 0, before training begins. 
If a sample’s loss exceeds this baseline during training, we assume the model is making a deliberate or persistent error, possibly due to label noise or input corruption. By capping the per-sample loss before computing the sampling distribution, we reduce the likelihood of repeatedly selecting such samples, thereby keeping the pruning process fairly robust.

\begin{figure}[t]
    \centering
    \includegraphics[width=1.0\linewidth]{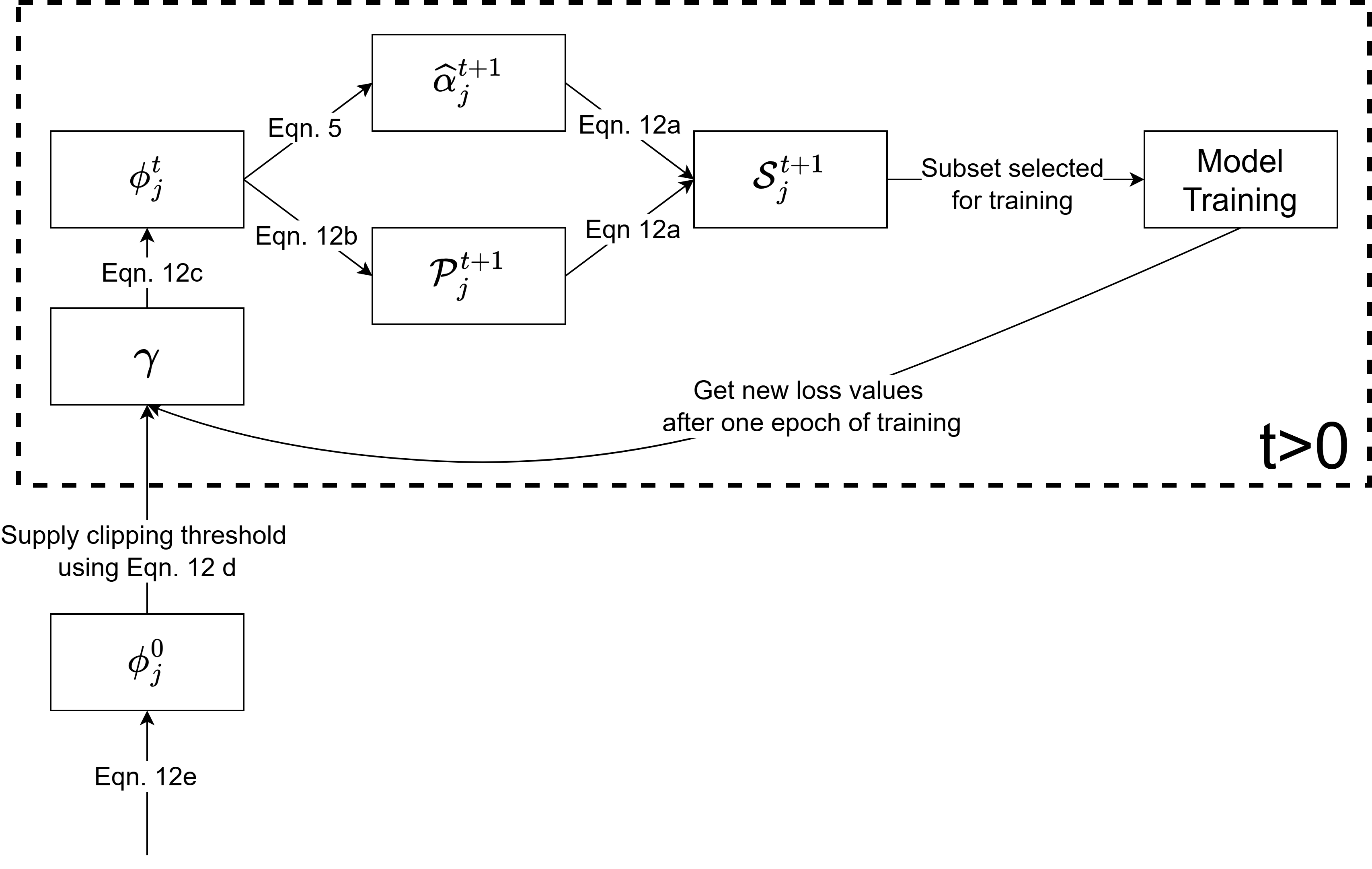}
    \caption{An overview of the sequence of steps involved in RCAP}
    \label{fig:RCAP_overview}
\end{figure}

Crucially, RCAP's per-class sample size determination and per-class sample selection modules incur no additional computational overhead, as they are determined entirely based on loss values that are computed during the forward pass. 
Such a strategy allows our proposed approach to achieve a per-sample time complexity of $O(1)$.
Algorithm \ref{alg:PDDP} provides the implementation details, where $I[j] = \{i ~|~ \forall ~y_i\in Y, y_i=j\}$ at $t=1$.

\newlength{\commentWidth}
\setlength{\commentWidth}{7cm}
\newcommand{\atcp}[1]{\tcp*[r]{\makebox[\commentWidth]{#1\hfill}}}
\begin{algorithm}
    \small
    \caption{The proposed RCAP Algorithm}
    \label{alg:PDDP}
    \SetKwInOut{Input}{Input}
    \SetKwInOut{Output}{Output}
    \Input{Dataset $\mathcal{S} = (\mathbf{X}, Y)$, Number of classes $c$, Pruning rate $r\in (0,1]$, Number of training epochs, $T$, Softmax temperature $\beta$ and Set of indices selected $I$}
    \Output{Trained Model.}
    $\alpha, m = [~], [~]$\\
    $n=$ length$(Y)$\\
    $\phi\leftarrow L(X_i) ~~\forall X_i\in\mathcal{S}$\\
    \For{$t = 1 \text{ to } T$} {
        \For{$j = 1 \text{ to } c$} {
            idx $\leftarrow\{i ~|~ \forall ~y_i\in Y, y_i=j\}$\\
            $m[j] = $ length$(\text{idx})$\\
            $\alpha[j]=\left\lfloor\frac{\sqrt{\frac{nc[j]}{n}\phi[I[j]]}}{\sum\sqrt{\frac{nc[j]}{n}\phi[I[j]]}}\times(1-r)\times\frac{n}{m[j]}\right\rfloor$\\
            Use Eqn. \ref{eqn:reweighting_alpha} to fix violating $\alpha$\\
            $\mathcal{P} = \frac{e^{\left(\phi[\text{idx}]/\beta\right)}}{\sum e^{\left(\phi[\text{idx}]/\beta\right)}}$\\
            I[j] $=\left\{i~|~\forall ~(X_i,y_i)\in \mathcal{X}\sim \mathcal{\mathcal{P}}\left(\mathbf{X}[idx], Y[idx]\right)\right\}$ and $|\mathcal{X}|=\left(\alpha[j]\times m[j]\right)$\\
        }
        Update model parameters using I\\
        Update $\phi$
    }
\end{algorithm}

\begin{table*}[t]
    \centering
    \smaller
    \caption{Worst Group Accuracy (Top-1) averaged over three separate runs. 
    The best scores are shown in bold while the second best are underlined.
    The time, in minutes, required by RCAP in comparison to full data training is also reported.}
    \label{tab:worst_group_results}
    \setlength\tabcolsep{4.2pt}
    \begin{tabular}{llllllllllr}
    \toprule
        Dataset & \shortstack[l]{Prune\\ Rate} & \multicolumn{1}{l}{CCS$\scriptstyle{(\%)}$} & \multicolumn{1}{l}{MetriQ$\scriptstyle{(\%)}$} & \multicolumn{1}{l}{TDDS$\scriptstyle{(\%)}$} & \multicolumn{1}{l}{UCB$\scriptstyle{(\%)}$} & \multicolumn{1}{l}{InfoBatch$\scriptstyle{(\%)}$} & \multicolumn{1}{l}{RS2 w/r$\scriptstyle{(\%)}$} & \multicolumn{1}{l}{RS2 w/o$\scriptstyle{(\%)}$} & \multicolumn{1}{l}{RCAP$\scriptstyle{(\%)}$} & \multicolumn{1}{c}{Time} \\
    \toprule
        \multirow{5}{*}{CIFAR10} & $00\%$ & $91.13 \scriptstyle{\pm 0.29}$ & $91.13\scriptstyle{\pm 0.29}$ & $91.13\scriptstyle{\pm 0.29}$ & $91.13\scriptstyle{\pm 0.29}$ & $91.13\scriptstyle{\pm 0.29}$ & $91.13\scriptstyle{\pm 0.29}$ & $91.13\scriptstyle{\pm 0.29}$ & $91.13\scriptstyle{\pm 0.29}$ & $23.3$ \\
        \cmidrule{2-11}
        & $50\%$ & $88.87\scriptstyle{\pm 0.29}$ & $\underline{90.53\scriptstyle{\pm 0.41}}$ & $90.27\scriptstyle{\pm 0.33}$ & $90.00\scriptstyle{\pm 0.45}$ & $89.97\scriptstyle{\pm 0.48}$ & $89.83\scriptstyle{\pm 0.24}$ & $90.10\scriptstyle{\pm 0.59}$ & $\mathbf{90.60\scriptstyle{\pm 0.14}}$ & $13.3$\\
        & $70\%$ & $83.50\scriptstyle{\pm 1.31}$ & $86.63\scriptstyle{\pm 0.29}$ & $84.57\scriptstyle{\pm 1.11}$ & $87.97\scriptstyle{\pm 0.45}$ & $88.50\scriptstyle{\pm 0.16}$ & $88.43\scriptstyle{\pm 0.49}$ & $\underline{88.60\scriptstyle{\pm 0.22}}$ & $\mathbf{89.73\scriptstyle{\pm 0.38}}$ & $6.7$\\
        & $80\%$ & $77.53\scriptstyle{\pm 0.87}$ & $82.50\scriptstyle{\pm 0.62}$ & $80.17\scriptstyle{\pm 1.11}$ & $84.53\scriptstyle{\pm 0.34}$ & $86.53\scriptstyle{\pm 0.97}$ & $87.43\scriptstyle{\pm 0.48}$ & $\underline{88.10\scriptstyle{\pm 0.50}}$ & $\mathbf{88.70}\scriptstyle{\pm 0.37}$ & $5.3$\\
        & $90\%$ & $67.20\scriptstyle{\pm 0.36}$ & $71.30\scriptstyle{\pm 1.08}$ & $68.63\scriptstyle{\pm 0.95}$ & $73.17\scriptstyle{\pm 0.38}$ & $\underline{83.53\scriptstyle{\pm 0.45}}$ & $79.63\scriptstyle{\pm 0.45}$ & $80.47\scriptstyle{\pm 0.33}$ & $\mathbf{85.07\scriptstyle{\pm 0.34}}$ & $3.3$\\
    \midrule
        \multirow{5}{*}{CIFAR100} & $00\%$ & $55.00\scriptstyle{\pm 1.41}$ & $55.00\scriptstyle{\pm 1.41}$ & $55.00\scriptstyle{\pm 1.41}$ & $55.00\scriptstyle{\pm 1.41}$ & $55.00\scriptstyle{\pm 1.41}$ & $55.00\scriptstyle{\pm 1.41}$ & $55.00\scriptstyle{\pm 1.41}$ & $55.00\scriptstyle{\pm 1.41}$ & $23.3$\\
        \cmidrule{2-11}
        & $50\%$ & $43.67\scriptstyle{\pm 0.47}$ & $\underline{54.00\scriptstyle{\pm 1.41}}$ & $43.67\scriptstyle{\pm 0.47}$ & $49.33\scriptstyle{\pm 0.47}$ & $52.33\scriptstyle{\pm 0.47}$ & $53.67\scriptstyle{\pm 1.25}$ & $\underline{54.00\scriptstyle{\pm 0.00}}$ & $\mathbf{55.00\scriptstyle{\pm 1.63}}$ & $13.3$\\
        & $70\%$ & $34.67\scriptstyle{\pm 0.47}$ & $43.67\scriptstyle{\pm 0.94}$ & $23.33\scriptstyle{\pm 1.70}$ & $42.57\scriptstyle{\pm 1.89}$ & $51.00\scriptstyle{\pm 1.63}$ & $50.33\scriptstyle{\pm 1.70}$ & $\underline{52.00\scriptstyle{\pm 0.82}}$ & $\mathbf{52.67\scriptstyle{\pm 0.94}}$ & $6.7$\\
        & $80\%$ & $21.00\scriptstyle{\pm 0.82}$ & $30.33\scriptstyle{\pm 0.47}$ & $15.67\scriptstyle{\pm 0.47}$ & $33.33\scriptstyle{\pm 1.89}$ & $\underline{50.33\scriptstyle{\pm 0.47}}$ & $\underline{50.33\scriptstyle{\pm 1.70}}$ & $49.00\scriptstyle{\pm 0.00}$ & $\mathbf{50.67\scriptstyle{\pm 0.47}}$ & $5.3$\\
        & $90\%$ & $07.00\scriptstyle{\pm 0.82}$ & $11.00\scriptstyle{\pm 1.63}$ & $05.33\scriptstyle{\pm 0.47}$ & $14.33\scriptstyle{\pm 1.25}$ & $\underline{46.67\scriptstyle{\pm 0.47}}$ & $35.67\scriptstyle{\pm 0.94}$ & $35.33\scriptstyle{\pm 0.94}$ & $\mathbf{48.33\scriptstyle{\pm 1.89}}$ & $3.3$\\
    \midrule
        \multirow{5}{*}{ImageNet} & $00\%$ & $20.67\scriptstyle{\pm 2.49}$ & $20.67\scriptstyle{\pm 2.49}$ & $20.67\scriptstyle{\pm 2.49}$ & $20.67\scriptstyle{\pm 2.49}$ & $20.67\scriptstyle{\pm 2.49}$ & $20.67\scriptstyle{\pm 2.49}$ & $20.67\scriptstyle{\pm 2.49}$ & $20.67\scriptstyle{\pm 2.49}$ & $249.0$\\
        \cmidrule{2-11}
        & $50\%$ & $00.00\scriptstyle{\pm 0.00}$ & $00.00\scriptstyle{\pm 0.00}$ & $00.00\scriptstyle{\pm 0.00}$ & $00.00\scriptstyle{\pm 0.00}$ & $18.67\scriptstyle{\pm 2.49}$ & $\underline{22.67\scriptstyle{\pm 0.94}}$ & $19.33\scriptstyle{\pm 0.94}$ & $\mathbf{24.00\scriptstyle{\pm 0.00}}$ & $130.5$\\
        & $70\%$ & $00.00\scriptstyle{\pm 0.00}$ & $00.00\scriptstyle{\pm 0.00}$ & $00.00\scriptstyle{\pm 0.00}$ & $00.00\scriptstyle{\pm 0.00}$ & $18.00\scriptstyle{\pm 2.83}$ & $\underline{20.00\scriptstyle{\pm 1.64}}$ & $\underline{20.00\scriptstyle{\pm 4.32}}$ & $\mathbf{24.00\scriptstyle{\pm 2.83}}$ & $82.7$\\
        & $80\%$ & $00.00\scriptstyle{\pm 0.00}$ & $00.00\scriptstyle{\pm 0.00}$ & $00.00\scriptstyle{\pm 0.00}$ & $00.00\scriptstyle{\pm 0.00}$ & $14.00\scriptstyle{\pm 0.00}$ & $20.67\scriptstyle{\pm 0.94}$ & $\underline{23.33\scriptstyle{\pm 0.94}}$ & $\mathbf{24.00\scriptstyle{\pm 1.63}}$ & $58.0$\\
        & $90\%$ & $00.00\scriptstyle{\pm 0.00}$ & $00.00\scriptstyle{\pm 0.00}$ & $00.00\scriptstyle{\pm 0.00}$ & $00.00\scriptstyle{\pm 0.00}$ & $20.00\scriptstyle{\pm 2.83}$ & $19.33\scriptstyle{\pm 1.89}$ & $\underline{23.33\scriptstyle{\pm 2.49}}$ & $\mathbf{26.00\scriptstyle{\pm 0.00}}$ & $30.5$\\
    \midrule
        \multirow{5}{*}{Waterbirds} & $00\%$ & $90.27\scriptstyle{\pm 0.67}$ & $90.27\scriptstyle{\pm 0.67}$ & $90.27\scriptstyle{\pm 0.67}$ & $90.27\scriptstyle{\pm 0.67}$ & $90.27\scriptstyle{\pm 0.67}$ & $90.27\scriptstyle{\pm 0.67}$ & $90.27\scriptstyle{\pm 0.67}$ & $90.27\scriptstyle{\pm 0.67}$ & $70.0$\\
        \cmidrule{2-11}
        & $50\%$ & $89.97\scriptstyle{\pm 0.47}$ & $89.22\scriptstyle{\pm 0.27}$ & $90.10\scriptstyle{\pm 0.35}$ & $50.00\scriptstyle{\pm 0.00}$ & $\underline{90.48\scriptstyle{\pm 0.71}}$ & $\underline{90.48\scriptstyle{\pm 0.17}}$ & $89.72\scriptstyle{\pm 1.42}$ & $\mathbf{91.34\scriptstyle{\pm 0.01}}$ & $35.0$\\
        & $70\%$ & $91.02\scriptstyle{\pm 0.06}$ & $82.35\scriptstyle{\pm 0.87}$ & $90.11\scriptstyle{\pm 0.18}$ & $50.00\scriptstyle{\pm 0.00}$ & $\underline{90.60\scriptstyle{\pm 0.61}}$ & $90.10\scriptstyle{\pm 0.18}$ & $90.35\scriptstyle{\pm 0.35}$ & $\mathbf{92.09\scriptstyle{\pm 0.09}}$ & $20.0$\\
        & $80\%$ & $90.40\scriptstyle{\pm 0.21}$ & $81.73\scriptstyle{\pm 0.26}$ & $\underline{90.71\scriptstyle{\pm 0.52}}$ & $50.00\scriptstyle{\pm 0.00}$ & $89.83\scriptstyle{\pm 0.25}$ & $89.61\scriptstyle{\pm 1.70}$ & $89.60\scriptstyle{\pm 0.94}$ & $\mathbf{91.60\scriptstyle{\pm 0.18}}$ & $15.0$\\
        & $90\%$ & $90.27\scriptstyle{\pm 0.04}$ & $79.05\scriptstyle{\pm 0.37}$ & $\underline{90.48\scriptstyle{\pm 0.18}}$ & $50.00\scriptstyle{\pm 0.00}$ & $89.06\scriptstyle{\pm 0.57}$ & $89.78\scriptstyle{\pm 0.35}$ & $88.97\scriptstyle{\pm 0.47}$ & $\mathbf{91.21\scriptstyle{\pm 0.38}}$ & $10.0$\\
    \midrule
        \multirow{5}{*}{CelebA} & $00\%$ & $90.14\scriptstyle{\pm 0.35}$ & $90.14\scriptstyle{\pm 0.35}$ & $90.14\scriptstyle{\pm 0.35}$ & $90.14\scriptstyle{\pm 0.35}$ & $90.14\scriptstyle{\pm 0.35}$ & $90.14\scriptstyle{\pm 0.35}$ & $90.14\scriptstyle{\pm 0.35}$ & $90.14\scriptstyle{\pm 0.35}$ & $21.7$\\
        \cmidrule{2-11}
        & $50\%$ & $86.43\scriptstyle{\pm 0.80}$ & \underline{$91.72\scriptstyle{\pm 1.40}$} & $87.44\scriptstyle{\pm 1.80}$ & $50.00\scriptstyle{\pm 0.00}$ & $86.47\scriptstyle{\pm 2.62}$ & $88.99\scriptstyle{\pm 0.91}$ & $87.97\scriptstyle{\pm 2.46}$ & $\mathbf{92.30\scriptstyle{\pm 0.16}}$ & $12.1$\\
        & $70\%$ & $86.28\scriptstyle{\pm 1.27}$ & \underline{$91.45\scriptstyle{\pm 0.30}$} & $88.29\scriptstyle{\pm 2.12}$ & $50.00\scriptstyle{\pm 0.00}$ & $84.46\scriptstyle{\pm 5.06}$ & $85.76\scriptstyle{\pm 2.21}$ & $88.27\scriptstyle{\pm 2.01}$ & $\mathbf{92.19\scriptstyle{\pm 0.22}}$ & $5.8$\\
        & $80\%$ & $89.58\scriptstyle{\pm 0.31}$ & \underline{$90.70\scriptstyle{\pm 0.26}$} & $86.16\scriptstyle{\pm 0.59}$ & $50.00\scriptstyle{\pm 0.00}$ & $82.17\scriptstyle{\pm 0.34}$ & $88.21\scriptstyle{\pm 1.58}$ & $88.96\scriptstyle{\pm 1.54}$ & $\mathbf{91.64\scriptstyle{\pm 0.57}}$ & $4.0$\\
        & $90\%$ & $84.75\scriptstyle{\pm 1.99}$ & \underline{$89.39\scriptstyle{\pm 1.48}$} & $80.57\scriptstyle{\pm 3.02}$ & $50.00\scriptstyle{\pm 0.00}$ & $79.35\scriptstyle{\pm 0.26}$ & $81.49\scriptstyle{\pm 1.72}$ & $84.25\scriptstyle{\pm 2.05}$ & $\mathbf{91.24\scriptstyle{\pm 0.41}}$ & $2.0$\\
    \midrule
        \multirow{5}{*}{iNaturalist} & $00\%$ & $69.66\scriptstyle{\pm1.17}$ & $69.66\scriptstyle{\pm1.17}$ & $69.66\scriptstyle{\pm1.17}$ & $69.66\scriptstyle{\pm1.17}$ & $69.66\scriptstyle{\pm1.17}$ & $69.66\scriptstyle{\pm1.17}$ & $69.66\scriptstyle{\pm1.17}$ & $69.66\scriptstyle{\pm1.17}$ & $58.5$\\
        \cmidrule{2-11}
        & $50\%$ & $65.62\scriptstyle{\pm0.13}$ & \underline{$65.97\scriptstyle{\pm0.48}$} & $49.32\scriptstyle{\pm1.79}$ & $0.00\scriptstyle{\pm0.00}$ & $62.76\scriptstyle{\pm0.26}$ & $61.73\scriptstyle{\pm1.46}$ & $63.01\scriptstyle{\pm1.37}$ & $\mathbf{66.44\scriptstyle{\pm2.06}}$ & $34.1$\\
        & $70\%$ & $61.12\scriptstyle{\pm0.12}$ & \underline{$65.94\scriptstyle{\pm0.82}$} & $48.61\scriptstyle{\pm1.39}$ & $00.00\scriptstyle{\pm0.00}$ & $40.42\scriptstyle{\pm8.91}$ & $51.37\scriptstyle{\pm2.05}$ & $54.79\scriptstyle{\pm0.00}$ & $\mathbf{69.18\scriptstyle{\pm2.06}}$ & $19.6$\\
        & $80\%$ & $61.53\scriptstyle{\pm2.09}$ & \underline{$65.70\scriptstyle{\pm0.70}$} & $26.05\scriptstyle{\pm1.74}$ & $00.00\scriptstyle{\pm0.00}$ & $36.31\scriptstyle{\pm6.17}$ & $40.42\scriptstyle{\pm4.8}$ & $37.68\scriptstyle{\pm0.69}$ & $\mathbf{69.18\scriptstyle{\pm0.69}}$ & $10.2$\\
        & $90\%$ & $56.64\scriptstyle{\pm0.89}$ & \underline{$65.14\scriptstyle{\pm4.31}$} & $00.00\scriptstyle{\pm0.00}$ & $00.00\scriptstyle{\pm0.00}$ & $05.48\scriptstyle{\pm4.11}$ & $03.41\scriptstyle{\pm0.71}$ & $00.69\scriptstyle{\pm0.69}$ & $\mathbf{68.49\scriptstyle{\pm2.74}}$ & $6.7$\\
    \bottomrule
    \end{tabular}
\end{table*}

\begin{table*}[t]
    \centering
    \smaller
    \caption{Average Group Accuracy (Top-1) averaged over three separate runs. 
    The best scores are shown in bold while the second best are underlined.
    The time, in minutes, required by RCAP in comparison to full data training is also reported.}
    \label{tab:average_group_results}
    \setlength\tabcolsep{3.8pt}
    \begin{tabular}{llllllllllr}
    \toprule
        Dataset & \shortstack[l]{Prune\\ Rate} & \multicolumn{1}{c}{CCS$\scriptstyle{(\%)}$} & \multicolumn{1}{c}{MetriQ$\scriptstyle{(\%)}$} & \multicolumn{1}{c}{TDDS$\scriptstyle{(\%)}$} & \multicolumn{1}{c}{UCB$\scriptstyle{(\%)}$} & \multicolumn{1}{c}{InfoBatch$\scriptstyle{(\%)}$} & \multicolumn{1}{c}{RS2 w/r$\scriptstyle{(\%)}$} & \multicolumn{1}{c}{RS2 w/o$\scriptstyle{(\%)}$} & \multicolumn{1}{c}{RCAP$\scriptstyle{(\%)}$} & \multicolumn{1}{c}{Time} \\
    \toprule
        \multirow{5}{*}{CIFAR10} & $00\%$ & $95.41\scriptstyle{\pm 0.10}$ & $95.41\scriptstyle{\pm 0.10}$ & $95.41\scriptstyle{\pm 0.10}$ & $95.41\scriptstyle{\pm 0.10}$ & $95.41\scriptstyle{\pm 0.10}$ & $95.41\scriptstyle{\pm 0.10}$ & $95.41\scriptstyle{\pm 0.10}$ & $95.41\scriptstyle{\pm 0.10}$ & 23.3\\\cmidrule{2-11}
        & $50\%$ & $93.85\scriptstyle{\pm 0.16}$ & $92.81\scriptstyle{\pm 0.19}$ & $\mathbf{94.88\scriptstyle{\pm 0.06}}$ & $94.78\scriptstyle{\pm 0.04}$ & $94.84\scriptstyle{\pm 0.04}$ & $\underline{94.85\scriptstyle{\pm 0.14}}$ & $94.64\scriptstyle{\pm 0.17}$ & $94.81\scriptstyle{\pm 0.24}$ & 13.3\\
        & $70\%$ & $90.10\scriptstyle{\pm 0.59}$ & $89.89\scriptstyle{\pm 0.40}$ & $91.79\scriptstyle{\pm 0.61}$ & $93.80\scriptstyle{\pm 0.17}$ & $94.14\scriptstyle{\pm 0.21}$ & $\underline{94.19\scriptstyle{\pm 0.16}}$ & $94.09\scriptstyle{\pm 0.39}$ & $\mathbf{94.40\scriptstyle{\pm 0.14}}$ & 6.7\\
        & $80\%$ & $86.50\scriptstyle{\pm 0.38}$ & $86.31\scriptstyle{\pm 0.54}$ & $89.59\scriptstyle{\pm 0.72}$ & $91.24\scriptstyle{\pm 0.28}$ & $\underline{93.44\scriptstyle{\pm 0.36}}$ & $93.22\scriptstyle{\pm 0.29}$ & $\mathbf{93.50\scriptstyle{\pm 0.32}}$ & $93.04\scriptstyle{\pm 0.16}$ & 5.3\\
        & $90\%$ & $81.20\scriptstyle{\pm 0.17}$ & $76.26\scriptstyle{\pm 1.00}$ & $83.60\scriptstyle{\pm 1.03}$ & $84.29\scriptstyle{\pm 0.83}$ & $\mathbf{91.83\scriptstyle{\pm 0.85}}$ & $88.82\scriptstyle{\pm 0.24}$ & $89.11\scriptstyle{\pm 0.81}$ & $\underline{91.45\scriptstyle{\pm 0.27}}$ & 3.3\\
    \midrule
        \multirow{5}{*}{CIFAR100} & $00\%$ & $77.85\scriptstyle{\pm 0.56}$ & $77.85\scriptstyle{\pm 0.56}$ & $77.85\scriptstyle{\pm 0.56}$ & $77.85\scriptstyle{\pm 0.56}$ & $77.85\scriptstyle{\pm 0.56}$ & $77.85\scriptstyle{\pm 0.56}$ & $77.85\scriptstyle{\pm 0.56}$ & $77.85\scriptstyle{\pm 0.56}$ & 23.3\\\cmidrule{2-11}
        & $50\%$ & $69.16\scriptstyle{\pm 0.51}$ & $69.50\scriptstyle{\pm 0.44}$ & $71.87\scriptstyle{\pm 0.54}$ & $74.95\scriptstyle{\pm 0.32}$ & $76.03\scriptstyle{\pm 0.47}$ & $76.35\scriptstyle{\pm 0.36}$ & $\underline{76.76\scriptstyle{\pm 0.38}}$& $\mathbf{76.90\scriptstyle{\pm 0.30}}$ & 13.3\\
        & $70\%$ & $64.85\scriptstyle{\pm 0.20}$ & $60.98\scriptstyle{\pm 0.36}$ & $64.94\scriptstyle{\pm 0.45}$ & $69.95\scriptstyle{\pm 0.98}$ & $75.53\scriptstyle{\pm 0.30}$ & $74.86\scriptstyle{\pm 0.59}$ & $\mathbf{75.76\scriptstyle{\pm 0.23}}$& $\underline{75.63\scriptstyle{\pm 0.13}}$ & 6.7\\
        & $80\%$ & $53.15\scriptstyle{\pm 1.30}$ & $51.03\scriptstyle{\pm 1.35}$ & $57.13\scriptstyle{\pm 1.33}$ & $62.13\scriptstyle{\pm 0.48}$ & $\mathbf{74.68\scriptstyle{\pm 0.11}}$ & $73.77\scriptstyle{\pm 0.71}$ & $73.68\scriptstyle{\pm 0.40}$ & $\underline{74.62\scriptstyle{\pm 0.13}}$ & 5.3\\
        & $90\%$ & $35.42\scriptstyle{\pm 0.72}$ & $29.99\scriptstyle{\pm 1.04}$ & $40.98\scriptstyle{\pm 1.69}$ & $40.03\scriptstyle{\pm 1.65}$ & $\mathbf{71.93\scriptstyle{\pm 0.38}}$ & $66.90\scriptstyle{\pm 0.29}$ & $66.59\scriptstyle{\pm 0.76}$ & $\underline{70.92\scriptstyle{\pm 0.22}}$ & 3.3\\
    \midrule
        \multirow{5}{*}{ImageNet} & $00\%$ & $84.47\scriptstyle{\pm 0.05}$ & $84.47\scriptstyle{\pm 0.05}$ & $84.47\scriptstyle{\pm 0.05}$ & $84.47\scriptstyle{\pm 0.05}$ & $84.47\scriptstyle{\pm 0.05}$ & $84.47\scriptstyle{\pm 0.05}$ & $84.47\scriptstyle{\pm 0.05}$ & $84.47\scriptstyle{\pm 0.05}$ & 249\\\cmidrule{2-11}
        & $50\%$ & $74.78\scriptstyle{\pm 0.19}$ & $71.78\scriptstyle{\pm 0.16}$ & $78.78\scriptstyle{\pm 0.27}$ & $80.23\scriptstyle{\pm 0.39}$ & $\mathbf{84.40\scriptstyle{\pm 0.03}}$ & $84.17\scriptstyle{\pm 0.09}$ & $84.31\scriptstyle{\pm 0.07}$ & $\underline{84.37\scriptstyle{\pm 0.04}}$ & 130.5\\
        & $70\%$ & $73.95\scriptstyle{\pm 0.21}$ & $70.95\scriptstyle{\pm 0.39}$ & $75.82\scriptstyle{\pm 0.72}$ & $77.29\scriptstyle{\pm 0.15}$ & $\underline{84.09\scriptstyle{\pm 0.18}}$ & $83.88\scriptstyle{\pm 0.11}$ & $\mathbf{84.14\scriptstyle{\pm 0.06}}$ & $83.89\scriptstyle{\pm 0.13}$ & 82.7\\
        & $80\%$ & $69.25\scriptstyle{\pm 0.19}$ & $70.98\scriptstyle{\pm 0.21}$ & $71.09\scriptstyle{\pm 0.44}$ & $72.43\scriptstyle{\pm 0.59}$ & $\underline{83.78\scriptstyle{\pm 0.05}}$ & $83.77\scriptstyle{\pm 0.03}$ & $\mathbf{83.94\scriptstyle{\pm 0.01}}$ & $83.46\scriptstyle{\pm 0.06}$ & 58.0\\
        & $90\%$ & $62.23\scriptstyle{\pm 0.41}$ & $70.16\scriptstyle{\pm 0.49}$ & $69.13\scriptstyle{\pm 0.63}$ & $71.24\scriptstyle{\pm 0.96}$ & $83.46\scriptstyle{\pm 0.06}$ & $83.19\scriptstyle{\pm 0.02}$ & $\underline{83.49\scriptstyle{\pm 0.03}}$ & $\mathbf{83.54\scriptstyle{\pm 0.02}}$ & 30.5\\
    \midrule
        \multirow{5}{*}{Waterbirds} & $00\%$ & $90.87\scriptstyle{\pm0.30}$ & $90.87\scriptstyle{\pm0.30}$ & $90.87\scriptstyle{\pm0.30}$ & $90.87\scriptstyle{\pm0.30}$ & $90.87\scriptstyle{\pm0.30}$ & $90.87\scriptstyle{\pm0.30}$ & $90.87\scriptstyle{\pm0.30}$ & $90.87\scriptstyle{\pm0.30}$ & 70.0\\\cmidrule{2-11}
        & $50\%$ & $90.84\scriptstyle{\pm 0.40}$ & $89.42\scriptstyle{\pm 0.24}$ & $90.71\scriptstyle{\pm 0.21}$ & $50.00\scriptstyle{\pm 0.00}$ & $\underline{91.65\scriptstyle{\pm 0.61}}$ & $91.15\scriptstyle{\pm 0.74}$ & $90.54\scriptstyle{\pm 1.08}$ & $\mathbf{91.78\scriptstyle{\pm 0.42}}$ & 35.0\\
        & $70\%$ & $91.40\scriptstyle{\pm 0.06}$ & $84.41\scriptstyle{\pm 1.62}$ & $90.98\scriptstyle{\pm 0.58}$ & $50.00\scriptstyle{\pm 0.00}$ & $\underline{91.42\scriptstyle{\pm 0.64}}$ & $90.30\scriptstyle{\pm 0.16}$ & $90.82\scriptstyle{\pm 0.36}$ & $\mathbf{92.26\scriptstyle{\pm 0.56}}$ & 20.0\\
        & $80\%$ & $90.61\scriptstyle{\pm 0.14}$ & $83.49\scriptstyle{\pm 0.96}$ & $\underline{90.93\scriptstyle{\pm 0.59}}$ & $50.00\scriptstyle{\pm 0.00}$ & $90.09\scriptstyle{\pm 0.18}$ & $90.01\scriptstyle{\pm 1.30}$ & $90.53\scriptstyle{\pm 0.29}$ & $\mathbf{91.98\scriptstyle{\pm 0.40}}$ & 15.0\\
        & $90\%$ & $90.33\scriptstyle{\pm 0.34}$ & $81.84\scriptstyle{\pm 0.34}$ & $\underline{90.77\scriptstyle{\pm 0.08}}$ & $50.00\scriptstyle{\pm 0.00}$ & $89.42\scriptstyle{\pm 0.72}$ & $89.99\scriptstyle{\pm 0.46}$ & $89.36\scriptstyle{\pm 0.35}$ & $\mathbf{91.47\scriptstyle{\pm 0.33}}$ & 10.0\\
    \midrule
        \multirow{5}{*}{CelebA} & $00\%$ & $92.04\scriptstyle{\pm 0.35}$ & $92.04\scriptstyle{\pm 0.35}$ & $92.04\scriptstyle{\pm 0.35}$ & $92.04\scriptstyle{\pm 0.35}$ & $92.04\scriptstyle{\pm 0.35}$ & $92.04\scriptstyle{\pm 0.35}$ & $92.04\scriptstyle{\pm 0.35}$ & $92.04\scriptstyle{\pm 0.35}$ & 21.7\\\cmidrule{2-11}
        & $50\%$ & $91.15\scriptstyle{\pm 0.09}$ & $\mathbf{93.14\scriptstyle{\pm 0.36}}$ & $91.45\scriptstyle{\pm 0.57}$ & $50.00\scriptstyle{\pm 0.00}$ & $91.28\scriptstyle{\pm 0.59}$ & $91.94\scriptstyle{\pm 0.26}$ & $90.87\scriptstyle{\pm 0.35}$ & \underline{$92.84\scriptstyle{\pm 0.36}$} & 12.1\\
        & $70\%$ & $91.28\scriptstyle{\pm 0.47}$ & \underline{$92.42\scriptstyle{\pm 0.05}$} & $91.10\scriptstyle{\pm 0.27}$ & $50.00\scriptstyle{\pm 0.00}$ & $90.19\scriptstyle{\pm 1.75}$ & $90.93\scriptstyle{\pm 0.67}$ & $91.03\scriptstyle{\pm 1.05}$ & $\mathbf{93.00\scriptstyle{\pm 0.07}}$ & 5.8\\
        & $80\%$ & $91.48\scriptstyle{\pm 0.09}$ & \underline{$91.68\scriptstyle{\pm 0.19}$} & $90.53\scriptstyle{\pm 0.38}$ & $50.00\scriptstyle{\pm 0.00}$ & $89.44\scriptstyle{\pm 0.09}$ & $90.20\scriptstyle{\pm 0.79}$ & $90.19\scriptstyle{\pm 0.54}$ & $\mathbf{92.39\scriptstyle{\pm 0.26}}$ & 4.0\\
        & $90\%$ & $87.08\scriptstyle{\pm 0.44}$ & \underline{$91.14\scriptstyle{\pm 0.62}$} & $87.46\scriptstyle{\pm 0.64}$ & $50.00\scriptstyle{\pm 0.00}$ & $87.79\scriptstyle{\pm 0.19}$ & $88.50\scriptstyle{\pm 0.48}$ & $89.38\scriptstyle{\pm 0.55}$ & $\mathbf{91.47\scriptstyle{\pm 0.31}}$ & 2.0\\
    \midrule
        \multirow{5}{*}{iNaturalist} & $00\%$ & $83.62\scriptstyle{\pm0.06}$ & $83.62\scriptstyle{\pm0.06}$ & $83.62\scriptstyle{\pm0.06}$ & $83.62\scriptstyle{\pm0.06}$ & $83.62\scriptstyle{\pm0.06}$ & $83.62\scriptstyle{\pm0.06}$ & $83.62\scriptstyle{\pm0.06}$ & $83.62\scriptstyle{\pm0.06}$ & $58.5$\\\cmidrule{2-11}
        & $50\%$ & $84.03\scriptstyle{\pm0.06}$ & \underline{$84.19\scriptstyle{\pm0.16}$} & $80.32\scriptstyle{\pm0.26}$ & $07.69\scriptstyle{\pm0.00}$ & $81.65\scriptstyle{\pm0.05}$ & $81.45\scriptstyle{\pm0.18}$ & $81.87\scriptstyle{\pm0.04}$ & $\mathbf{84.26\scriptstyle{\pm0.09}}$ & $34.1$\\
        & $70\%$ & $81.86\scriptstyle{\pm1.14}$ & \underline{$82.50\scriptstyle{\pm0.17}$} & $76.12\scriptstyle{\pm0.48}$ & $07.69\scriptstyle{\pm0.00}$ & $75.40\scriptstyle{\pm0.74}$ & $78.82\scriptstyle{\pm0.33}$ & $79.85\scriptstyle{\pm0.15}$ & $\mathbf{84.12\scriptstyle{\pm0.36}}$ & $19.6$\\
        & $80\%$ & $80.25\scriptstyle{\pm0.38}$ & \underline{$82.81\scriptstyle{\pm0.36}$} & $70.20\scriptstyle{\pm0.30}$ & $07.69\scriptstyle{\pm0.00}$ & $76.73\scriptstyle{\pm0.03}$ & $76.31\scriptstyle{\pm0.16}$ & $76.43\scriptstyle{\pm0.86}$ & $\mathbf{83.93\scriptstyle{\pm0.24}}$ & $10.2$\\
        & $90\%$ & \underline{$79.74\scriptstyle{\pm0.16}$} & $78.32\scriptstyle{\pm4.50}$ & $63.19\scriptstyle{\pm0.14}$ & $07.69\scriptstyle{\pm0.00}$ & $70.67\scriptstyle{\pm0.85}$ & $70.40\scriptstyle{\pm0.32}$ & $70.34\scriptstyle{\pm0.76}$ & $\mathbf{82.97\scriptstyle{\pm0.57}}$ & $6.7$\\
    \bottomrule
    \end{tabular}
\end{table*}

\section{Experiments}
\subsection{Baselines}
We evaluate our proposed approach against seven representative baselines: two static and four dynamic data pruning methods as well as a coreset selection technique.
\textbf{CCS} \citep{DBLP:conf/iclr/ZhengLL023} is a state-of-the-art coreset selection technique that maximizes data distribution coverage. 
\textbf{TDDS} \citep{DBLP:conf/cvpr/ZhangDLXZ24} is the current leading static data pruning method. 
It incorporates training dynamics to determine sample importance. 
\textbf{MetriQ} \citep{DBLP:journals/corr/abs-2404-05579} is a class-ratio-aware static data pruning method designed to reduce classification bias.
\textbf{UCB} \cite{raju2021accelerating} is one of the earliest dynamic data pruning approaches utilizing sample uncertainty and Reinforcement Learning inspired exploration to prune unimportant samples.
\textbf{InfoBatch} \citep{qin2024infobatch} is a state-of-the-art dynamic data pruning method that selects samples based on their loss and adaptively determines the pruning ratio via a hyperparameter.
\textbf{RS2} \citep{DBLP:conf/iclr/OkanovicWMNKKGR24} is another state-of-the-art dynamic data pruning approach that performs pruning by random selection with and without replacement.
\textbf{Note:} For brevity's sake, we omit comparisons with older methods (e.g., GraNd, CRAIG, GradMatch, Glister, and CREST) as all considered baselines have demonstrated superior performance in prior studies.

\subsection{Dataset and Model Details}
We benchmark RCAP on six diverse datasets in terms of scale and class imbalance using five distinct networks.
\textbf{CIFAR10} \citep{krizhevsky2009learning} is a medium-scale, class-balanced dataset comprising $10$ classes, each containing $5000$ samples over which we trained the ResNet18 model \citep{DBLP:conf/cvpr/HeZRS16} from scratch.
\textbf{CIFAR100} \citep{krizhevsky2009learning} is a medium-scale, class-balanced dataset comprising $100$ classes, each containing $500$ samples, over which we trained the ResNet18  model from scratch as well.
\textbf{ImageNet} \citep{5206848} is a large-scale, relatively class-balanced dataset of over $1.2$ million images comprising $1000$ classes, each containing approximately $1300$ samples with slight variations. 
We trained a two layer MLP on top of the Dinov2-b model \citep{DBLP:conf/cvpr/RadosavovicKGHD20} on this dataset.
\textbf{Waterbirds} \citep{DBLP:journals/corr/abs-1911-08731} is a moderately class-imbalanced, small scale dataset with 4795 images split into land birds and water birds $(76.8\% vs. 23.2\%)$. 
We fine-tuned the EfficientNet-b3 model \citep{pmlr-v97-tan19a} pre-trained on ImageNet.
\textbf{CelebA} \citep{DBLP:conf/iccv/LiuLWT15} is a relatively high class-imbalanced, medium-scale dataset containing over $160K$ images. 
We chose the blonde $85.1\%$ vs not blonde $14.9\%$, binary classification task.
We train an EfficientFormerV2 \citep{li2023rethinking} from scratch, for this dataset. 
\textbf{iNaturalist} \citep{Horn_2018_CVPR} is a large-scale, extremely imbalanced dataset with over $600K$ images across $13$ superclasses. 
The largest group contains $196,613$ images, while the smallest has $381$. 
We fine-tune an ImageNet pre-trained ResNet50 \citep{DBLP:conf/cvpr/HeZRS16}.

\subsection{Training Details}
To ensure fair evaluation, we re-implement all baselines and verify that the Top-1 average accuracy of each method matches its corresponding reported value. 
For robustness analysis, we report the worst group accuracy and corresponding average group accuracy (Top-1). 
In doing so, the average group accuracy of each method as reported in their corresponding manuscripts changes considerably, especially at high pruning rates. 
All static methods utilize the same network for subset selection and training which is the best-case scenario for such techniques.
For a fair comparison between InfoBatch and other baselines, we adjust the number of training iterations as recommended by \cite{qin2024infobatch}.
To understand RCAP's training efficiency, we report its training time across all pruning rates as well as the total time for full dataset training, in minutes. 
Further training specifics are detailed in Section \ref{sec:supp-training-details} of the Supplementary Material.

\subsection{Results}
Table \ref{tab:worst_group_results} presents the Top-1 worst-group accuracy across six datasets at four pruning rates. 
RCAP consistently outperforms all baselines in every experiment. 
Notably, on ImageNet, Waterbirds and CelebA, it surpasses full-data training, even at a $90\%$ pruning rate.
The improvement stems from the class imbalance in these datasets, ranging from mild to high. 
By pruning aggressively, RCAP naturally retains fewer samples from the majority class while preserving most or all minority-class samples. 
This results in a more balanced classification task, which enhances worst-group accuracy. 
Other pruning methods like MetriQ and CCS also benefit from the same effect. 
We find that on large-scale imbalanced datasets with few classes, particularly CelebA and iNaturalist, static pruning methods such as CCS and MetriQ perform significantly better than dynamic pruning techniques.
However, such methods fail entirely, in terms of worst group accuracy, on ImageNet where the number of classes is large. 
Note that all static pruning techniques are evaluated in their best-case scenario, i.e., the architecture used for data pruning and consequent training on the retained data are the same. 
We find that UCB, which was originally only tested on CIFAR10 and CIFAR100, performs the worst among all baselines with the resultant models producing random output for four out of six datasets.
Table \ref{tab:average_group_results} reports Top-1 average-group accuracy. 
RCAP performs comparably to existing methods on class-balanced datasets. 
On class-imbalanced datasets, it outperforms all baselines and even full-data training, primarily due to its gains in worst-group accuracy.

Beyond accuracy, RCAP is highly efficient. 
It delivers up to $8.69\times$ speed-up in comparison to full-data training with less than $1\%$ drop in performance on ImageNet, Waterbirds, CelebA and iNaturalist datasets, on average, as demonstrated in Table \ref{tab:worst_group_results} (or Table \ref{tab:average_group_results}).
This combination of efficiency and robustness establishes RCAP as the new state-of-the-art in robust dynamic dataset pruning.

\begin{figure*}[t]
    \centering
    \subfloat[CIFAR10]{\includegraphics[width=0.25\textwidth]{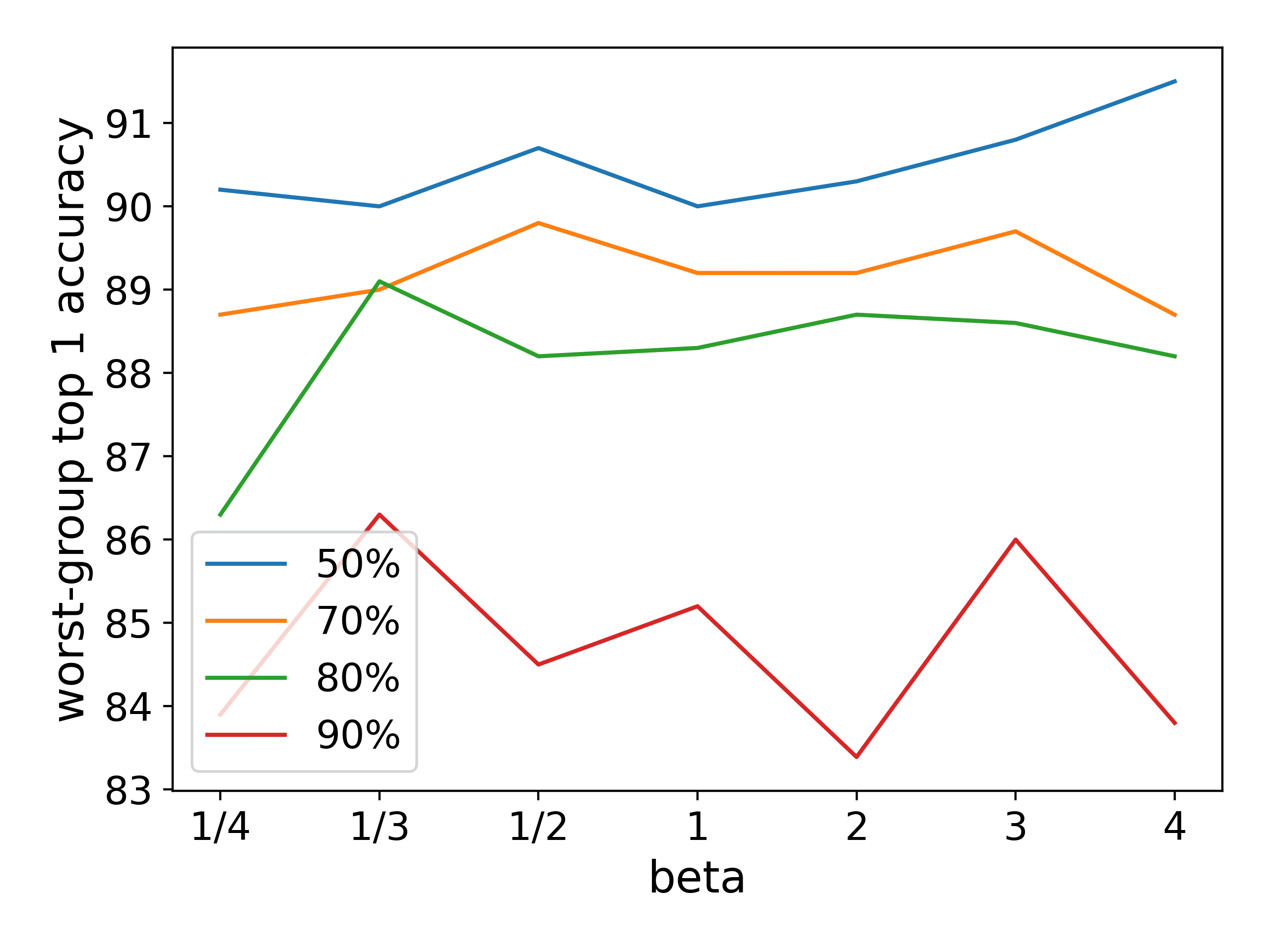}\label{fig:cifar10}} \hfill
    \subfloat[CIFAR100]{\includegraphics[width=0.25\textwidth]{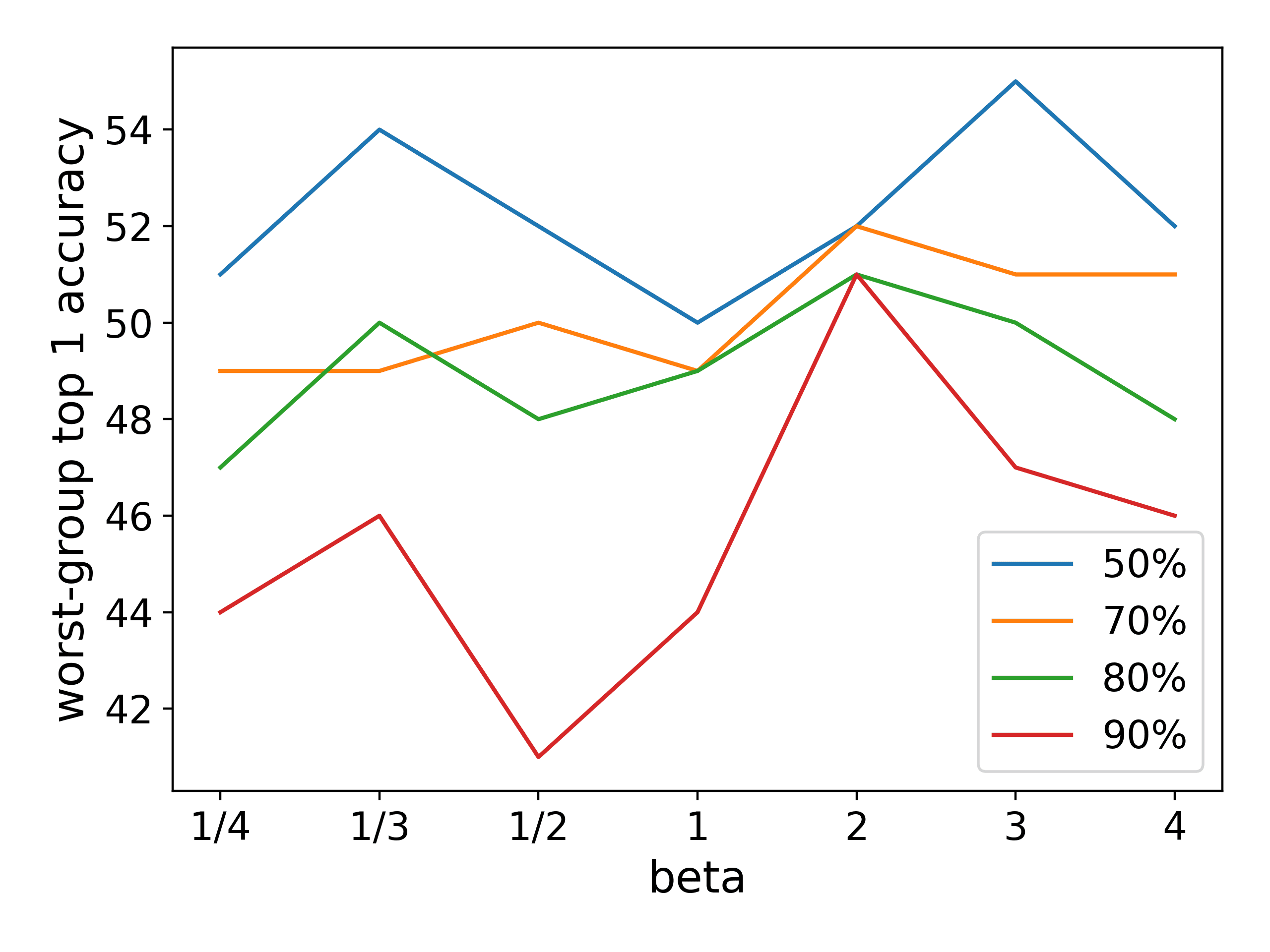}\label{fig:cifar100}} \hfill
    \subfloat[Waterbirds]{\includegraphics[width=0.25\textwidth]{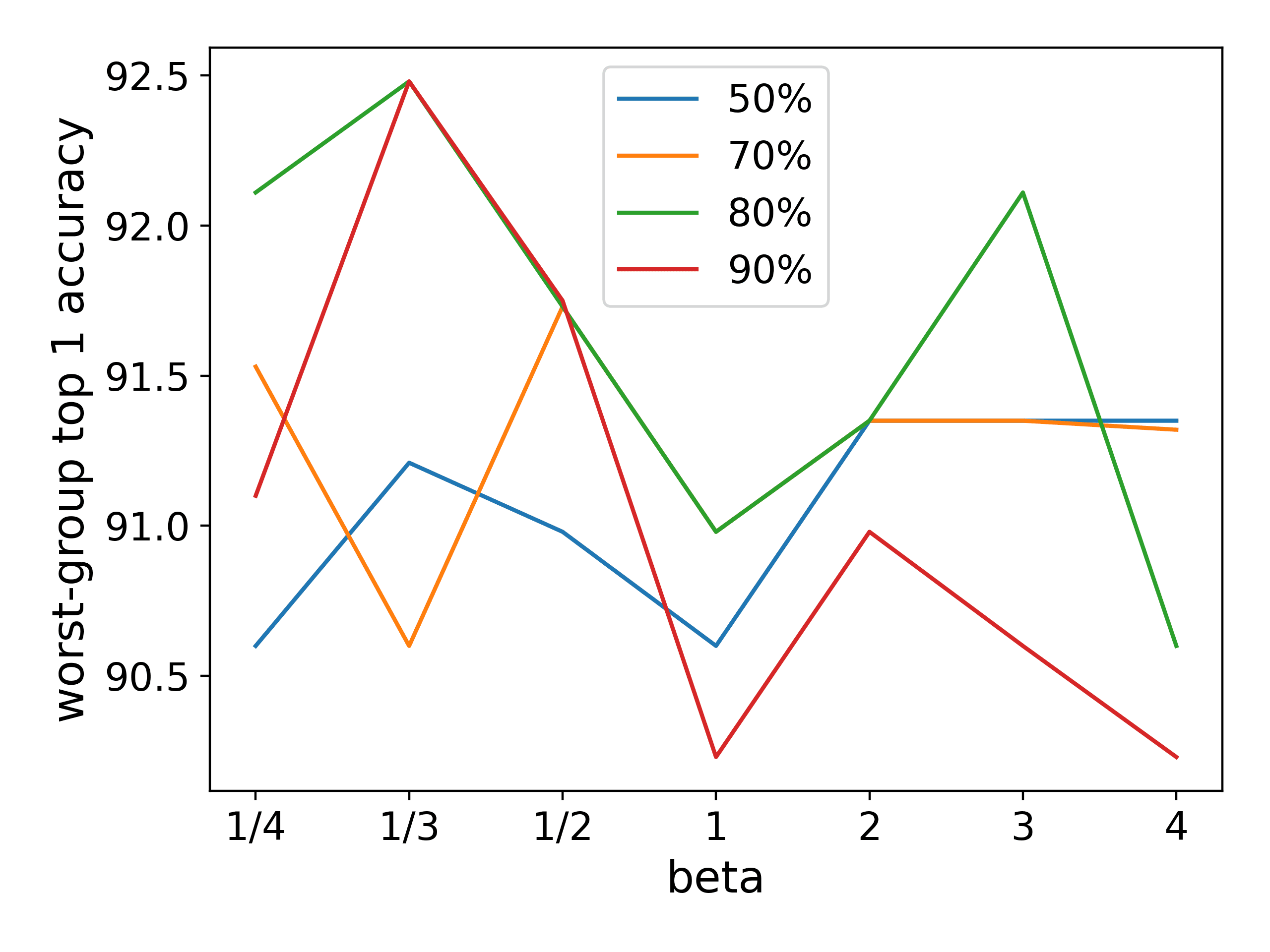}\label{fig:waterbirds}} \hfill
    \subfloat[CelebA]{\includegraphics[width=0.25\textwidth]{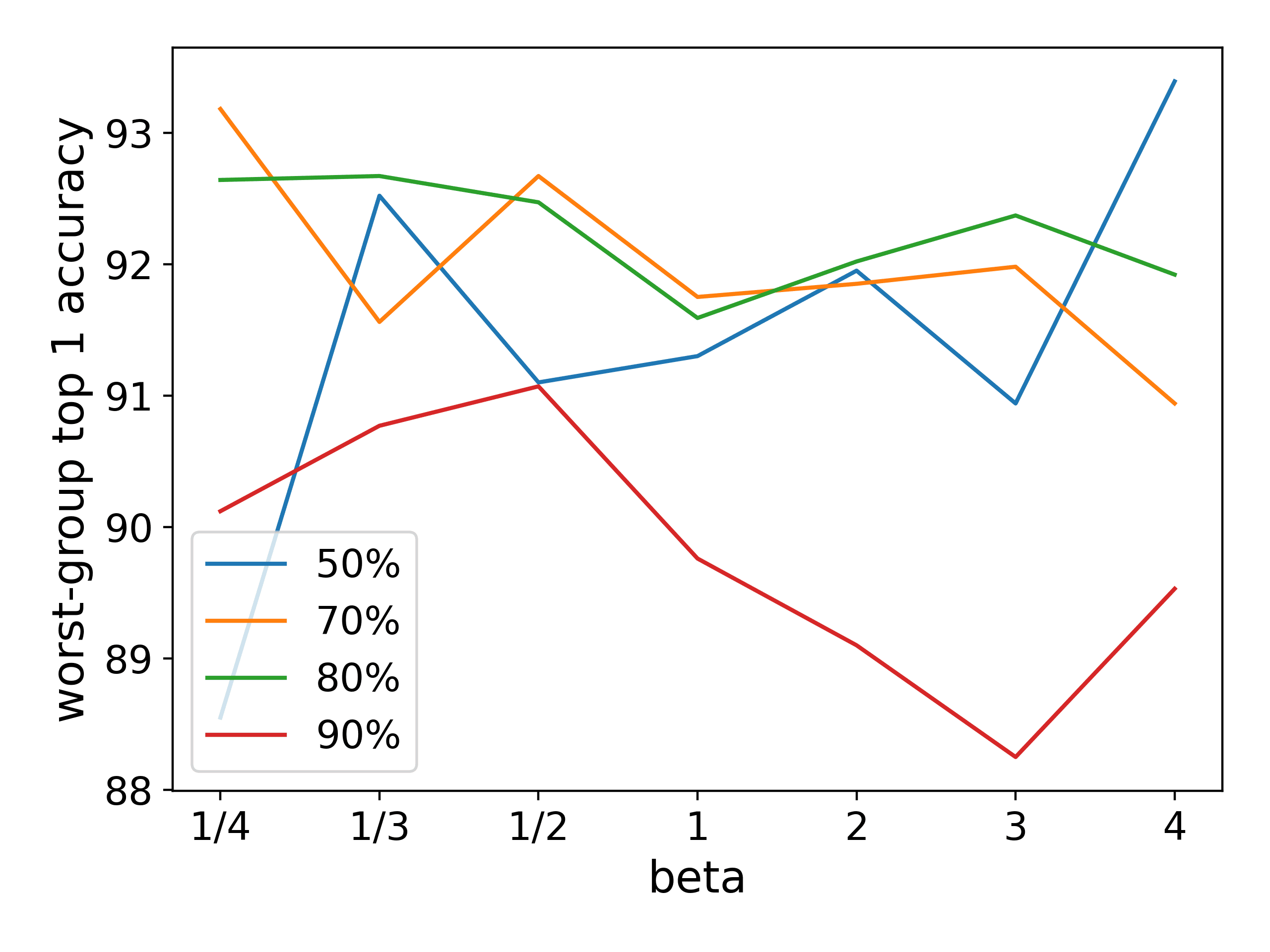}\label{fig:celeba}}
    \caption{Variation of the Softmax temperature hyper-parameter, $\beta$, across different pruning rates over four datasets: CIFAR10, CIFAR100, Waterbirds, and CelebA.}
    \label{fig:beta_ablation}
\end{figure*}

\subsection{Performance with Varying $\beta$}
The Softmax temperature hyper-parameter, $\beta$, is the sole hyper-parameter in RCAP and plays a critical role in determining the sampling probabilities, thereby influencing the overall performance of our algorithm. 
Specifically, $\beta$ controls the sharpness of the sampling distribution with values $>1$ promoting a more uniform sampling strategy, while values $<1$ prioritizing samples having high loss by assigning them higher sampling probabilities.
To evaluate the impact of $\beta$ on RCAP's performance, we conduct an ablation study by varying $\beta$ within the range $[\frac{1}{4}, 4]$ across four datasets, CIFAR10, CIFAR100, Waterbirds, and CelebA over four different pruning rates, $50\%$, $70\%$, $80\%$, and $90\%$. 
The results, presented in Fig. \ref{fig:beta_ablation}, reveal that the choice of $\beta$ is crucial, especially at higher pruning rates where fewer samples are retained.
Interestingly, at relatively moderate pruning rates $(50\%)$, $\beta>1$ performs better with the performance being more stable across a wider range of $\beta$, allowing for more flexibility in hyper-parameter selection. 
On the other hand, at high pruning rates $(90\%)$, $\beta<1$ achieves better results with the choice of $\beta$ becoming increasingly critical as fewer samples are retained. 
Suboptimal values lead to sharp performance drops, particularly in Waterbirds and CIFAR100.
These findings suggest that $\beta$ must be carefully tuned based on dataset characteristics and pruning rates. 
We recommend that starting with $\beta = \left\{\frac{1}{3}, \frac{1}{2}, 2, 3\right\}$ provides a strong baseline across most scenarios.

\section{Related Work}\label{sec:related}
Dataset pruning algorithms aim to identify and remove less "informative" examples, minimizing the performance gap between models trained on subsets and full datasets. 
This is achieved through various sample importance estimation metrics that measure the amount of "information" imparted by an example during model training.

\textbf{Geometry based methods} reduce redundancy by leveraging spatial similarity in feature space. 
Popular techniques like Herding \citep{DBLP:conf/icml/Welling09} minimize the distance between coreset and dataset centers, while K-Center Greedy \citep{DBLP:conf/iclr/SenerS18} minimizes the maximum distance to the nearest coreset sample.\\
\textbf{Uncertainty based methods} prioritize low-confidence samples using metrics like least confidence, entropy, and margin \citep{DBLP:conf/iclr/ColemanYMMBLLZ20}. 
For example, \cite{DBLP:conf/nips/ChangLM17} use predictive distribution variance for selection.\\
\textbf{Loss based methods} focus on samples contributing higher loss or gradient values. 
Forgetting events \citep{DBLP:conf/iclr/TonevaSCTBG19}, GraNd, and EL2N scores \citep{paul2021deep} are prominent techniques. 
CCS \citep{DBLP:conf/iclr/ZhengLL023}, a state-of-the-art one-shot coreset selection technique maximizes data distribution coverage while utilizing stratified sampling to form the retain set.
InfoBatch \citep{qin2024infobatch}, a dynamic pruning approach, combines loss thresholds with score based sampling and uniform sampling along with gradient scaling for bias reduction.\\
\textbf{Decision boundary based methods} prioritize samples near decision boundaries. 
Adversarial DeepFool \citep{DBLP:journals/corr/abs-1802-09841} measures perturbations needed to alter predictions, while CAL \citep{DBLP:conf/emnlp/MargatinaVBA21} emphasizes predictive divergence among neighbors.\\
\textbf{Gradient matching based methods} optimize coresets to approximate full-dataset gradients. 
CRAIG \citep{DBLP:conf/icml/MirzasoleimanBL20} and GradMatch \citep{DBLP:conf/icml/KillamsettySRDI21} minimize gradient error, with GradMatch introducing penalties to prevent over-reliance on few samples.\\
\textbf{Bilevel optimization based methods} frame sample selection as an optimization problem. 
Retrieve \citep{DBLP:conf/nips/KillamsettyZCI21} applies this to semi-supervised learning, while Glister \citep{DBLP:conf/aaai/KillamsettySRI21} introduces robustness by adding a validation set on the outer optimization and the log-likelihood in the bilevel optimization.\\
\textbf{Training dynamics incorporating methods} track importance over epochs. 
Dyn-Unc \citep{DBLP:conf/cvpr/HeY0Z22} averages prediction uncertainty across epochs, and TDDS \citep{DBLP:conf/cvpr/ZhangDLXZ24} aligns gradients over the full training run.\\
\textbf{Submodularity based methods} maximize diversity and informativeness through submodular functions like Facility Location and Log Determinant \citep{DBLP:conf/alt/IyerKBA21}. 
Prism \citep{DBLP:journals/corr/abs-2103-00128} targets labeling efficiency in large datasets.\\
\textbf{Proxy based methods} train a proxy model (a smaller, shallower version of the original model) on the entire training dataset to determine the importance of each sample \citep{DBLP:conf/iclr/ColemanYMMBLLZ20,sachdeva2021svp}.\\
\textbf{Random sampling based methods} perform uniform sampling and are tough-to-beat baselines \citep{DBLP:journals/tmlr/AyedH23}. 
RS2 \citep{DBLP:conf/iclr/OkanovicWMNKKGR24} employs dynamic uniform sampling (with and without replacement), while MetriQ \citep{DBLP:journals/corr/abs-2404-05579} adjusts sample fractions by class.

\section{Conclusion}
We present RCAP, a novel, Robust, Class-Aware, Probabilistic dynamic dataset pruning algorithm tailored for classification tasks. 
In every epoch, RCAP applies a closed-form solution to estimate the fraction of samples that need to be included in the training subset for each individual class.
Thereafter, RCAP employs a novel, adaptive sampling strategy that prioritizes samples having a higher loss for populating the class-wise subset.
Our method incurs no computational overhead, achieving an impressive $8.69\times$ speed-up on average across multiple datasets while maintaining $<1\%$ drop in performance with respect to full data training. 
Extensive evaluation on six datasets, ranging from class-balanced to highly imbalanced, across four pruning rates and three distinct training paradigms, shows that RCAP significantly improves worst-group accuracy while maintaining competitive average-group accuracy compared to seven state-of-the-art pruning methods.\\
\textbf{Limitations and Future Scope:} 
RCAP requires a few training epochs to accurately approximate the optimal class-wise fractions, which could hinder its utility in scenarios requiring immediate effectiveness. 
For instance, LLMs are few-shot learners \citep{DBLP:conf/nips/BrownMRSKDNSSAA20} but are computationally expensive to train due to their reliance on massive datasets. 
Therefore, we aim to refine RCAP to reduce approximation error early in training, enhancing its applicability to LLMs and other large-scale models. 
Another important limitation is the $\beta$ hyper-parameter, which currently requires manual selection. 
A promising direction of future work is to make $\beta$ adaptive during training by utilizing the loss values or annealing schedules. 

\bibliography{uai2025-template}

@article{raju2021accelerating,
  title={Accelerating deep learning with dynamic data pruning},
  author={Raju, Ravi S and Daruwalla, Kyle and Lipasti, Mikko},
  journal={arXiv preprint arXiv:2111.12621},
  year={2021}
}

@inproceedings{
    okanovic2024repeated,
    title     = {Repeated Random Sampling for Minimizing the Time-to-Accuracy of Learning},
    author    = {Patrik Okanovic and Roger Waleffe and Vasilis Mageirakos and Konstantinos Nikolakakis and Amin Karbasi and Dionysios Kalogerias and Nezihe Merve G{\"u}rel and Theodoros Rekatsinas},
    booktitle = {The Twelfth International Conference on Learning Representations},
    year      = {2024},
}

@inproceedings{
  qin2024infobatch,
  title={InfoBatch: Lossless Training Speed Up by Unbiased Dynamic Data Pruning},
  author={Qin, Ziheng and Wang, Kai and Zheng, Zangwei and Gu, Jianyang and Peng, Xiangyu and Zhaopan Xu and Zhou, Daquan and Lei Shang and Baigui Sun and Xuansong Xie and You, Yang},
  booktitle={The Twelfth International Conference on Learning Representations},
  year={2024},
}

@article{paul2021deep,
  title={Deep learning on a data diet: Finding important examples early in training},
  author={Paul, Mansheej and Ganguli, Surya and Dziugaite, Gintare Karolina},
  journal={Advances in neural information processing systems},
  volume={34},
  pages={20596--20607},
  year={2021}
}

@inproceedings{DBLP:conf/iclr/0006XP0S023,
  author       = {Shuo Yang and
                  Zeke Xie and
                  Hanyu Peng and
                  Min Xu and
                  Mingming Sun and
                  Ping Li},
  title        = {Dataset Pruning: Reducing Training Data by Examining Generalization
                  Influence},
  booktitle    = {The Eleventh International Conference on Learning Representations,
                  {ICLR} 2023, Kigali, Rwanda, May 1-5, 2023},
  publisher    = {OpenReview.net},
  year         = {2023},
  bibsource    = {dblp computer science bibliography, https://dblp.org}
}

@inproceedings{DBLP:conf/cvpr/ZhangDLXZ24,
  author       = {Xin Zhang and
                  Jiawei Du and
                  Yunsong Li and
                  Weiying Xie and
                  Joey Tianyi Zhou},
  title        = {Spanning Training Progress: Temporal Dual-Depth Scoring {(TDDS)} for
                  Enhanced Dataset Pruning},
  booktitle    = {{IEEE/CVF} Conference on Computer Vision and Pattern Recognition,
                  {CVPR} 2024, Seattle, WA, USA, June 16-22, 2024},
  pages        = {26213--26222},
  publisher    = {{IEEE}},
  year         = {2024},
  doi          = {10.1109/CVPR52733.2024.02477},
  bibsource    = {dblp computer science bibliography, https://dblp.org}
}

@inproceedings{DBLP:conf/icml/YangCGZLZN24,
  author       = {Shuo Yang and
                  Zhe Cao and
                  Sheng Guo and
                  Ruiheng Zhang and
                  Ping Luo and
                  Shengping Zhang and
                  Liqiang Nie},
  title        = {Mind the Boundary: Coreset Selection via Reconstructing the Decision
                  Boundary},
  booktitle    = {Forty-first International Conference on Machine Learning, {ICML} 2024,
                  Vienna, Austria, July 21-27, 2024},
  publisher    = {OpenReview.net},
  year         = {2024},
  bibsource    = {dblp computer science bibliography, https://dblp.org}
}

@inproceedings{DBLP:conf/iclr/SenerS18,
  author       = {Ozan Sener and
                  Silvio Savarese},
  title        = {Active Learning for Convolutional Neural Networks: {A} Core-Set Approach},
  booktitle    = {6th International Conference on Learning Representations, {ICLR} 2018,
                  Vancouver, BC, Canada, April 30 - May 3, 2018, Conference Track Proceedings},
  publisher    = {OpenReview.net},
  year         = {2018},
  bibsource    = {dblp computer science bibliography, https://dblp.org}
}

@inproceedings{DBLP:conf/icml/Welling09,
  author       = {Max Welling},
  title        = {Herding dynamical weights to learn},
  booktitle    = {Proceedings of the 26th Annual International Conference on Machine
                  Learning, {ICML} 2009, Montreal, Quebec, Canada, June 14-18, 2009},
  series       = {{ACM} International Conference Proceeding Series},
  volume       = {382},
  pages        = {1121--1128},
  publisher    = {{ACM}},
  year         = {2009},
  doi          = {10.1145/1553374.1553517},
  bibsource    = {dblp computer science bibliography, https://dblp.org}
}

@inproceedings{DBLP:conf/iclr/ColemanYMMBLLZ20,
  author       = {Cody Coleman and
                  Christopher Yeh and
                  Stephen Mussmann and
                  Baharan Mirzasoleiman and
                  Peter Bailis and
                  Percy Liang and
                  Jure Leskovec and
                  Matei Zaharia},
  title        = {Selection via Proxy: Efficient Data Selection for Deep Learning},
  booktitle    = {8th International Conference on Learning Representations, {ICLR} 2020,
                  Addis Ababa, Ethiopia, April 26-30, 2020},
  publisher    = {OpenReview.net},
  year         = {2020},
  bibsource    = {dblp computer science bibliography, https://dblp.org}
}

@article{sachdeva2021svp,
  title={Svp-cf: Selection via proxy for collaborative filtering data},
  author={Sachdeva, Noveen and Wu, Carole-Jean and McAuley, Julian},
  journal={arXiv preprint arXiv:2107.04984},
  year={2021}
}

@inproceedings{DBLP:conf/cvpr/HeY0Z22,
  author       = {Muyang He and
                  Shuo Yang and
                  Tiejun Huang and
                  Bo Zhao},
  title        = {Large-scale Dataset Pruning with Dynamic Uncertainty},
  booktitle    = {{IEEE/CVF} Conference on Computer Vision and Pattern Recognition,
                  {CVPR} 2024 - Workshops, Seattle, WA, USA, June 17-18, 2024},
  pages        = {7713--7722},
  publisher    = {{IEEE}},
  year         = {2024},
  doi          = {10.1109/CVPRW63382.2024.00767},
  bibsource    = {dblp computer science bibliography, https://dblp.org}
}

@article{DBLP:journals/corr/abs-1802-09841,
  author       = {Melanie Ducoffe and
                  Fr{\'{e}}d{\'{e}}ric Precioso},
  title        = {Adversarial Active Learning for Deep Networks: a Margin Based Approach},
  journal      = {CoRR},
  volume       = {abs/1802.09841},
  year         = {2018},
  eprinttype    = {arXiv},
  eprint       = {1802.09841},
  bibsource    = {dblp computer science bibliography, https://dblp.org}
}

@inproceedings{DBLP:conf/emnlp/MargatinaVBA21,
  author       = {Katerina Margatina and
                  Giorgos Vernikos and
                  Lo{\"{\i}}c Barrault and
                  Nikolaos Aletras},
  title        = {Active Learning by Acquiring Contrastive Examples},
  booktitle    = {Proceedings of the 2021 Conference on Empirical Methods in Natural
                  Language Processing, {EMNLP} 2021, Virtual Event / Punta Cana, Dominican
                  Republic, 7-11 November, 2021},
  pages        = {650--663},
  publisher    = {Association for Computational Linguistics},
  year         = {2021},
  doi          = {10.18653/V1/2021.EMNLP-MAIN.51},
  bibsource    = {dblp computer science bibliography, https://dblp.org}
}

@inproceedings{DBLP:conf/nips/ChangLM17,
  author       = {Haw{-}Shiuan Chang and
                  Erik G. Learned{-}Miller and
                  Andrew McCallum},
  title        = {Active Bias: Training More Accurate Neural Networks by Emphasizing
                  High Variance Samples},
  booktitle    = {Advances in Neural Information Processing Systems 30: Annual Conference
                  on Neural Information Processing Systems 2017, December 4-9, 2017,
                  Long Beach, CA, {USA}},
  pages        = {1002--1012},
  year         = {2017},
  bibsource    = {dblp computer science bibliography, https://dblp.org}
}

@inproceedings{DBLP:conf/iclr/TonevaSCTBG19,
  author       = {Mariya Toneva and
                  Alessandro Sordoni and
                  Remi Tachet des Combes and
                  Adam Trischler and
                  Yoshua Bengio and
                  Geoffrey J. Gordon},
  title        = {An Empirical Study of Example Forgetting during Deep Neural Network
                  Learning},
  booktitle    = {7th International Conference on Learning Representations, {ICLR} 2019,
                  New Orleans, LA, USA, May 6-9, 2019},
  publisher    = {OpenReview.net},
  year         = {2019},
  bibsource    = {dblp computer science bibliography, https://dblp.org}
}

@inproceedings{DBLP:conf/icml/MirzasoleimanBL20,
  author       = {Baharan Mirzasoleiman and
                  Jeff A. Bilmes and
                  Jure Leskovec},
  title        = {Coresets for Data-efficient Training of Machine Learning Models},
  booktitle    = {Proceedings of the 37th International Conference on Machine Learning,
                  {ICML} 2020, 13-18 July 2020, Virtual Event},
  series       = {Proceedings of Machine Learning Research},
  volume       = {119},
  pages        = {6950--6960},
  publisher    = {{PMLR}},
  year         = {2020},
  bibsource    = {dblp computer science bibliography, https://dblp.org}
}

@inproceedings{DBLP:conf/icml/KillamsettySRDI21,
  author       = {KrishnaTeja Killamsetty and
                  Durga Sivasubramanian and
                  Ganesh Ramakrishnan and
                  Abir De and
                  Rishabh K. Iyer},
  title        = {{GRAD-MATCH:} Gradient Matching based Data Subset Selection for Efficient
                  Deep Model Training},
  booktitle    = {Proceedings of the 38th International Conference on Machine Learning,
                  {ICML} 2021, 18-24 July 2021, Virtual Event},
  series       = {Proceedings of Machine Learning Research},
  volume       = {139},
  pages        = {5464--5474},
  publisher    = {{PMLR}},
  year         = {2021},
  bibsource    = {dblp computer science bibliography, https://dblp.org}
}

@inproceedings{DBLP:conf/nips/KillamsettyZCI21,
  author       = {KrishnaTeja Killamsetty and
                  Xujiang Zhao and
                  Feng Chen and
                  Rishabh K. Iyer},
  title        = {{RETRIEVE:} Coreset Selection for Efficient and Robust Semi-Supervised
                  Learning},
  booktitle    = {Advances in Neural Information Processing Systems 34: Annual Conference
                  on Neural Information Processing Systems 2021, NeurIPS 2021, December
                  6-14, 2021, virtual},
  pages        = {14488--14501},
  year         = {2021},
  bibsource    = {dblp computer science bibliography, https://dblp.org}
}

@inproceedings{DBLP:conf/aaai/KillamsettySRI21,
  author       = {KrishnaTeja Killamsetty and
                  Durga Sivasubramanian and
                  Ganesh Ramakrishnan and
                  Rishabh K. Iyer},
  title        = {{GLISTER:} Generalization based Data Subset Selection for Efficient
                  and Robust Learning},
  booktitle    = {Thirty-Fifth {AAAI} Conference on Artificial Intelligence, {AAAI}
                  2021, Thirty-Third Conference on Innovative Applications of Artificial
                  Intelligence, {IAAI} 2021, The Eleventh Symposium on Educational Advances
                  in Artificial Intelligence, {EAAI} 2021, Virtual Event, February 2-9,
                  2021},
  pages        = {8110--8118},
  publisher    = {{AAAI} Press},
  year         = {2021},
  doi          = {10.1609/AAAI.V35I9.16988},
  bibsource    = {dblp computer science bibliography, https://dblp.org}
}

@inproceedings{DBLP:conf/alt/IyerKBA21,
  author       = {Rishabh K. Iyer and
                  Ninad Khargoankar and
                  Jeff A. Bilmes and
                  Himanshu Asanani},
  title        = {Submodular combinatorial information measures with applications in
                  machine learning},
  booktitle    = {Algorithmic Learning Theory, 16-19 March 2021, Virtual Conference,
                  Worldwide},
  series       = {Proceedings of Machine Learning Research},
  volume       = {132},
  pages        = {722--754},
  publisher    = {{PMLR}},
  year         = {2021},
  bibsource    = {dblp computer science bibliography, https://dblp.org}
}

@article{DBLP:journals/corr/abs-2103-00128,
  author       = {Vishal Kaushal and
                  Suraj Kothawade and
                  Ganesh Ramakrishnan and
                  Jeff A. Bilmes and
                  Rishabh K. Iyer},
  title        = {{PRISM:} {A} Unified Framework of Parameterized Submodular Information
                  Measures for Targeted Data Subset Selection and Summarization},
  journal      = {CoRR},
  volume       = {abs/2103.00128},
  year         = {2021},
  eprinttype    = {arXiv},
  eprint       = {2103.00128},
  bibsource    = {dblp computer science bibliography, https://dblp.org}
}

@inproceedings{DBLP:conf/iclr/OkanovicWMNKKGR24,
  author       = {Patrik Okanovic and
                  Roger Waleffe and
                  Vasilis Mageirakos and
                  Konstantinos E. Nikolakakis and
                  Amin Karbasi and
                  Dionysios S. Kalogerias and
                  Nezihe Merve G{\"{u}}rel and
                  Theodoros Rekatsinas},
  title        = {Repeated Random Sampling for Minimizing the Time-to-Accuracy of Learning},
  booktitle    = {The Twelfth International Conference on Learning Representations,
                  {ICLR} 2024, Vienna, Austria, May 7-11, 2024},
  publisher    = {OpenReview.net},
  year         = {2024},
  bibsource    = {dblp computer science bibliography, https://dblp.org}
}

@article{DBLP:journals/tmlr/AyedH23,
  author       = {Fadhel Ayed and
                  Soufiane Hayou},
  title        = {Data pruning and neural scaling laws: fundamental limitations of score-based
                  algorithms},
  journal      = {Trans. Mach. Learn. Res.},
  volume       = {2023},
  year         = {2023},
  bibsource    = {dblp computer science bibliography, https://dblp.org}
}

@article{DBLP:journals/corr/abs-2404-05579,
  author       = {Artem Vysogorets and
                  Kartik Ahuja and
                  Julia Kempe},
  title        = {Robust Data Pruning: Uncovering and Overcoming Implicit Bias},
  journal      = {CoRR},
  volume       = {abs/2404.05579},
  year         = {2024},
  doi          = {10.48550/ARXIV.2404.05579},
  eprinttype    = {arXiv},
  eprint       = {2404.05579},
  bibsource    = {dblp computer science bibliography, https://dblp.org}
}

@article{krizhevsky2009learning,
  title={Learning multiple layers of features from tiny images},
  author={Krizhevsky, Alex and Hinton, Geoffrey and others},
  year={2009},
  publisher={Toronto, ON, Canada}
}

@inproceedings{DBLP:conf/iccv/LiuLWT15,
  author       = {Ziwei Liu and
                  Ping Luo and
                  Xiaogang Wang and
                  Xiaoou Tang},
  title        = {Deep Learning Face Attributes in the Wild},
  booktitle    = {2015 {IEEE} International Conference on Computer Vision, {ICCV} 2015,
                  Santiago, Chile, December 7-13, 2015},
  pages        = {3730--3738},
  publisher    = {{IEEE} Computer Society},
  year         = {2015},
  doi          = {10.1109/ICCV.2015.425},
  bibsource    = {dblp computer science bibliography, https://dblp.org}
}

@article{DBLP:journals/corr/abs-1911-08731,
  author       = {Shiori Sagawa and
                  Pang Wei Koh and
                  Tatsunori B. Hashimoto and
                  Percy Liang},
  title        = {Distributionally Robust Neural Networks for Group Shifts: On the Importance
                  of Regularization for Worst-Case Generalization},
  journal      = {CoRR},
  volume       = {abs/1911.08731},
  year         = {2019},
  eprinttype    = {arXiv},
  eprint       = {1911.08731},
  bibsource    = {dblp computer science bibliography, https://dblp.org}
}

@inproceedings{DBLP:conf/cvpr/HeZRS16,
  author       = {Kaiming He and
                  Xiangyu Zhang and
                  Shaoqing Ren and
                  Jian Sun},
  title        = {Deep Residual Learning for Image Recognition},
  booktitle    = {2016 {IEEE} Conference on Computer Vision and Pattern Recognition,
                  {CVPR} 2016, Las Vegas, NV, USA, June 27-30, 2016},
  pages        = {770--778},
  publisher    = {{IEEE} Computer Society},
  year         = {2016},
  doi          = {10.1109/CVPR.2016.90},
  bibsource    = {dblp computer science bibliography, https://dblp.org}
}

@inproceedings{DBLP:conf/iccv/MullerH21,
  author       = {Samuel G. M{\"{u}}ller and
                  Frank Hutter},
  title        = {TrivialAugment: Tuning-free Yet State-of-the-Art Data Augmentation},
  booktitle    = {2021 {IEEE/CVF} International Conference on Computer Vision, {ICCV}
                  2021, Montreal, QC, Canada, October 10-17, 2021},
  pages        = {754--762},
  publisher    = {{IEEE}},
  year         = {2021},
  doi          = {10.1109/ICCV48922.2021.00081},
  bibsource    = {dblp computer science bibliography, https://dblp.org}
}

@inproceedings{DBLP:conf/iclr/DosovitskiyB0WZ21,
  author       = {Alexey Dosovitskiy and
                  Lucas Beyer and
                  Alexander Kolesnikov and
                  Dirk Weissenborn and
                  Xiaohua Zhai and
                  Thomas Unterthiner and
                  Mostafa Dehghani and
                  Matthias Minderer and
                  Georg Heigold and
                  Sylvain Gelly and
                  Jakob Uszkoreit and
                  Neil Houlsby},
  title        = {An Image is Worth 16x16 Words: Transformers for Image Recognition
                  at Scale},
  booktitle    = {9th International Conference on Learning Representations, {ICLR} 2021,
                  Virtual Event, Austria, May 3-7, 2021},
  publisher    = {OpenReview.net},
  year         = {2021},
  bibsource    = {dblp computer science bibliography, https://dblp.org}
}

@inproceedings{DBLP:conf/icml/RadfordKXBMS23,
  author       = {Alec Radford and
                  Jong Wook Kim and
                  Tao Xu and
                  Greg Brockman and
                  Christine McLeavey and
                  Ilya Sutskever},
  title        = {Robust Speech Recognition via Large-Scale Weak Supervision},
  booktitle    = {International Conference on Machine Learning, {ICML} 2023, 23-29 July
                  2023, Honolulu, Hawaii, {USA}},
  series       = {Proceedings of Machine Learning Research},
  volume       = {202},
  pages        = {28492--28518},
  publisher    = {{PMLR}},
  year         = {2023},
  bibsource    = {dblp computer science bibliography, https://dblp.org}
}

@inproceedings{DBLP:conf/nips/BaevskiZMA20,
  author       = {Alexei Baevski and
                  Yuhao Zhou and
                  Abdelrahman Mohamed and
                  Michael Auli},
  title        = {wav2vec 2.0: {A} Framework for Self-Supervised Learning of Speech
                  Representations},
  booktitle    = {Advances in Neural Information Processing Systems 33: Annual Conference
                  on Neural Information Processing Systems 2020, NeurIPS 2020, December
                  6-12, 2020, virtual},
  year         = {2020},
  bibsource    = {dblp computer science bibliography, https://dblp.org}
}

@inproceedings{DBLP:conf/nips/BrownMRSKDNSSAA20,
  author       = {Tom B. Brown and
                  Benjamin Mann and
                  Nick Ryder and
                  Melanie Subbiah and
                  Jared Kaplan and
                  Prafulla Dhariwal and
                  Arvind Neelakantan and
                  Pranav Shyam and
                  Girish Sastry and
                  Amanda Askell and
                  Sandhini Agarwal and
                  Ariel Herbert{-}Voss and
                  Gretchen Krueger and
                  Tom Henighan and
                  Rewon Child and
                  Aditya Ramesh and
                  Daniel M. Ziegler and
                  Jeffrey Wu and
                  Clemens Winter and
                  Christopher Hesse and
                  Mark Chen and
                  Eric Sigler and
                  Mateusz Litwin and
                  Scott Gray and
                  Benjamin Chess and
                  Jack Clark and
                  Christopher Berner and
                  Sam McCandlish and
                  Alec Radford and
                  Ilya Sutskever and
                  Dario Amodei},
  title        = {Language Models are Few-Shot Learners},
  booktitle    = {Advances in Neural Information Processing Systems 33: Annual Conference
                  on Neural Information Processing Systems 2020, NeurIPS 2020, December
                  6-12, 2020, virtual},
  year         = {2020},
  bibsource    = {dblp computer science bibliography, https://dblp.org}
}

@article{radford2019language,
  title={Language models are unsupervised multitask learners},
  author={Radford, Alec and Wu, Jeffrey and Child, Rewon and Luan, David and Amodei, Dario and Sutskever, Ilya and others},
  journal={OpenAI blog},
  volume={1},
  number={8},
  pages={9},
  year={2019}
}

@article{DBLP:journals/corr/abs-2303-08774,
  author       = {OpenAI},
  title        = {{GPT-4} Technical Report},
  journal      = {CoRR},
  volume       = {abs/2303.08774},
  year         = {2023},
  doi          = {10.48550/ARXIV.2303.08774},
  eprinttype    = {arXiv},
  eprint       = {2303.08774},
  bibsource    = {dblp computer science bibliography, https://dblp.org}
}

@inproceedings{DBLP:conf/icml/MindermannBRS0X22,
  author       = {S{\"{o}}ren Mindermann and
                  Jan Markus Brauner and
                  Muhammed Razzak and
                  Mrinank Sharma and
                  Andreas Kirsch and
                  Winnie Xu and
                  Benedikt H{\"{o}}ltgen and
                  Aidan N. Gomez and
                  Adrien Morisot and
                  Sebastian Farquhar and
                  Yarin Gal},
  title        = {Prioritized Training on Points that are Learnable, Worth Learning,
                  and not yet Learnt},
  booktitle    = {International Conference on Machine Learning, {ICML} 2022, 17-23 July
                  2022, Baltimore, Maryland, {USA}},
  series       = {Proceedings of Machine Learning Research},
  volume       = {162},
  pages        = {15630--15649},
  publisher    = {{PMLR}},
  year         = {2022},
  bibsource    = {dblp computer science bibliography, https://dblp.org}
}

@inproceedings{DBLP:conf/icml/XiaLZWWL24,
  author       = {Xiaobo Xia and
                  Jiale Liu and
                  Shaokun Zhang and
                  Qingyun Wu and
                  Hongxin Wei and
                  Tongliang Liu},
  title        = {Refined Coreset Selection: Towards Minimal Coreset Size under Model
                  Performance Constraints},
  booktitle    = {Forty-first International Conference on Machine Learning, {ICML} 2024,
                  Vienna, Austria, July 21-27, 2024},
  publisher    = {OpenReview.net},
  year         = {2024},
  bibsource    = {dblp computer science bibliography, https://dblp.org}
}

@inproceedings{DBLP:conf/wacv/ZhaoB23,
  author       = {Bo Zhao and
                  Hakan Bilen},
  title        = {Dataset Condensation with Distribution Matching},
  booktitle    = {{IEEE/CVF} Winter Conference on Applications of Computer Vision, {WACV}
                  2023, Waikoloa, HI, USA, January 2-7, 2023},
  pages        = {6503--6512},
  publisher    = {{IEEE}},
  year         = {2023},
  doi          = {10.1109/WACV56688.2023.00645},
  bibsource    = {dblp computer science bibliography, https://dblp.org}
}

@inproceedings{DBLP:conf/cvpr/Cazenavette00EZ22a,
  author       = {George Cazenavette and
                  Tongzhou Wang and
                  Antonio Torralba and
                  Alexei A. Efros and
                  Jun{-}Yan Zhu},
  title        = {Dataset Distillation by Matching Training Trajectories},
  booktitle    = {{IEEE/CVF} Conference on Computer Vision and Pattern Recognition Workshops,
                  {CVPR} Workshops 2022, New Orleans, LA, USA, June 19-20, 2022},
  pages        = {4749--4758},
  publisher    = {{IEEE}},
  year         = {2022},
  doi          = {10.1109/CVPRW56347.2022.00521},
  bibsource    = {dblp computer science bibliography, https://dblp.org}
}

@InProceedings{pmlr-v97-tan19a,
  title = 	 {{E}fficient{N}et: Rethinking Model Scaling for Convolutional Neural Networks},
  author =       {Tan, Mingxing and Le, Quoc},
  booktitle = 	 {Proceedings of the 36th International Conference on Machine Learning},
  pages = 	 {6105--6114},
  year = 	 {2019},
  volume = 	 {97},
  series = 	 {Proceedings of Machine Learning Research},
  month = 	 {09--15 Jun},
  publisher =    {PMLR},
}

@inproceedings{DBLP:conf/iclr/ZhengLL023,
  author       = {Haizhong Zheng and
                  Rui Liu and
                  Fan Lai and
                  Atul Prakash},
  title        = {Coverage-centric Coreset Selection for High Pruning Rates},
  booktitle    = {The Eleventh International Conference on Learning Representations,
                  {ICLR} 2023, Kigali, Rwanda, May 1-5, 2023},
  publisher    = {OpenReview.net},
  year         = {2023},
  bibsource    = {dblp computer science bibliography, https://dblp.org}
}

@INPROCEEDINGS{5206848,
  author={Deng, Jia and Dong, Wei and Socher, Richard and Li, Li-Jia and Kai Li and Li Fei-Fei},
  booktitle={2009 IEEE Conference on Computer Vision and Pattern Recognition}, 
  title={ImageNet: A large-scale hierarchical image database}, 
  year={2009},
  volume={},
  number={},
  pages={248-255},
  doi={10.1109/CVPR.2009.5206848}}

@InProceedings{Horn_2018_CVPR,
author = {
Van Horn, Grant and Mac Aodha, Oisin and Song, Yang and Cui, Yin and Sun, Chen
and Shepard, Alex and Adam, Hartwig and Perona, Pietro and Belongie, Serge},
title = {The INaturalist Species Classification and Detection Dataset},
booktitle = {
The IEEE Conference on Computer Vision and Pattern Recognition (CVPR)},
month = {June},
year = {2018}
}

@inproceedings{DBLP:conf/cvpr/RadosavovicKGHD20,
  author       = {Ilija Radosavovic and
                  Raj Prateek Kosaraju and
                  Ross B. Girshick and
                  Kaiming He and
                  Piotr Doll{\'{a}}r},
  title        = {Designing Network Design Spaces},
  booktitle    = {2020 {IEEE/CVF} Conference on Computer Vision and Pattern Recognition,
                  {CVPR} 2020, Seattle, WA, USA, June 13-19, 2020},
  pages        = {10425--10433},
  publisher    = {Computer Vision Foundation / {IEEE}},
  year         = {2020},
  doi          = {10.1109/CVPR42600.2020.01044},
  bibsource    = {dblp computer science bibliography, https://dblp.org}
}

@inproceedings{li2023rethinking,
  title={Rethinking vision transformers for mobilenet size and speed},
  author={Li, Yanyu and Hu, Ju and Wen, Yang and Evangelidis, Georgios and Salahi, Kamyar and Wang, Yanzhi and Tulyakov, Sergey and Ren, Jian},
  booktitle={Proceedings of the IEEE/CVF International Conference on Computer Vision},
  pages={16889--16900},
  year={2023}
}

@inproceedings{loshchilov2022sgdr,
  title={SGDR: Stochastic Gradient Descent with Warm Restarts},
  author={Loshchilov, Ilya and Hutter, Frank},
  booktitle={International Conference on Learning Representations},
  year={2022}
}

@article{DBLP:journals/corr/LoshchilovH15,
  author       = {Ilya Loshchilov and
                  Frank Hutter},
  title        = {Online Batch Selection for Faster Training of Neural Networks},
  journal      = {CoRR},
  volume       = {abs/1511.06343},
  year         = {2015},
  url          = {http://arxiv.org/abs/1511.06343},
  eprinttype    = {arXiv},
  eprint       = {1511.06343},
  bibsource    = {dblp computer science bibliography, https://dblp.org}
}

@inproceedings{DBLP:journals/corr/SchaulQAS15,
  author       = {Tom Schaul and
                  John Quan and
                  Ioannis Antonoglou and
                  David Silver},
  editor       = {Yoshua Bengio and
                  Yann LeCun},
  title        = {Prioritized Experience Replay},
  booktitle    = {4th International Conference on Learning Representations, {ICLR} 2016,
                  San Juan, Puerto Rico, May 2-4, 2016, Conference Track Proceedings},
  year         = {2016},
  url          = {http://arxiv.org/abs/1511.05952},
  bibsource    = {dblp computer science bibliography, https://dblp.org}
}

@article{DBLP:journals/corr/abs-1910-00762,
  author       = {Angela H. Jiang and
                  Daniel L.{-}K. Wong and
                  Giulio Zhou and
                  David G. Andersen and
                  Jeffrey Dean and
                  Gregory R. Ganger and
                  Gauri Joshi and
                  Michael Kaminsky and
                  Michael Kozuch and
                  Zachary C. Lipton and
                  Padmanabhan Pillai},
  title        = {Accelerating Deep Learning by Focusing on the Biggest Losers},
  journal      = {CoRR},
  volume       = {abs/1910.00762},
  year         = {2019},
  url          = {http://arxiv.org/abs/1910.00762},
  eprinttype    = {arXiv},
  eprint       = {1910.00762},
  bibsource    = {dblp computer science bibliography, https://dblp.org}
}

\newpage

\onecolumn

\title{RCAP: Robust, Class-Aware, Probabilistic Dynamic Dataset Pruning\\(Supplementary Material)}
\maketitle
\appendix
\section{Additional Simulation Results}\label{sec:simulation}
To further support our argument, we train a small feed-forward neural network using the Adam optimizer on a toy, three-class classification dataset consisting of $90$ examples and provide a plot of the loss and gradient norm of $10$ randomly selected examples over a $100$ epoch training run. 
The figure demonstrates that there indeed is a monotonic relation.

\begin{figure*}[h]
    \centering
    \includegraphics[width=0.9\linewidth]{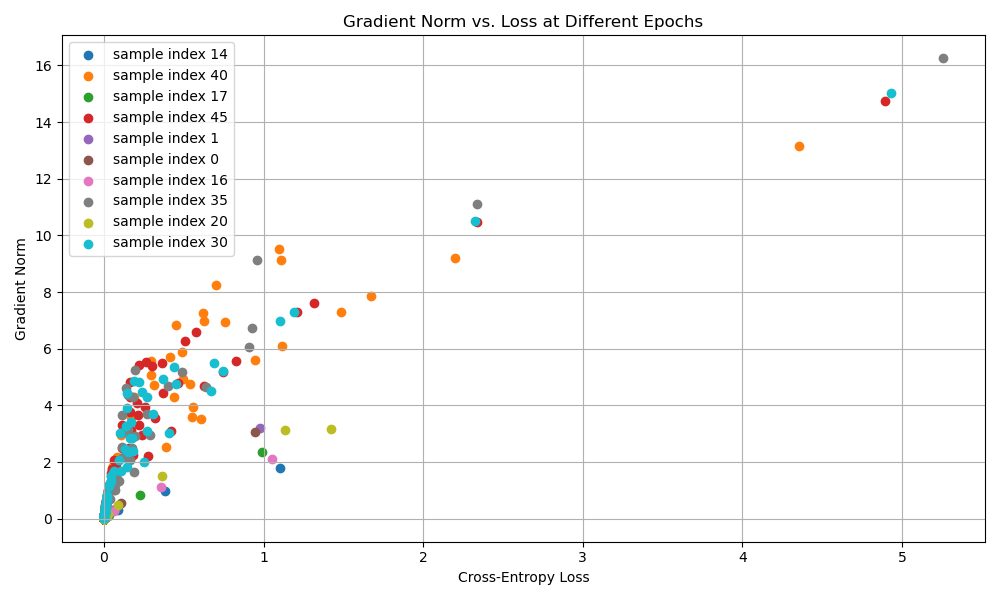}
    \caption{Visualizing the relationship between cross-entropy loss against gradient norm.}
    \label{fig:loss_vs_grad_norm}
\end{figure*}

\section{Training Details}\label{sec:supp-training-details}
We run all our tasks on a single NVIDIA A100 GPU in combination with an Intel Xeon processor. 
We use the Pytorch Lightning library to implement all methods. 
Each reported result is averaged over three different runs usings seeds, $0,27,100$.
Apart from standard image augmentations, we also employ TrivialAugmentWide \cite{DBLP:conf/iccv/MullerH21}. 
In all our experiments, we use the CosineAnnealing Scheduler \cite{loshchilov2022sgdr}.
\begin{table}[ht]
    \smaller
    \centering
    \caption{All the training details required to reproduce our results.}
    \label{tab:reproduce}
    \begin{tabular}{cccccccc}
    \toprule
        Dataset & Model & Augmentations & Optimizer & LR & Weight Decay & Batch Size & Epochs\\
    \toprule
        CIFAR10 & ResNet18 & \shortstack[c]{RandomCrop\\ RandomHorizontalFLip} & \shortstack[c]{SGD \\ momentum$=0.9$} & $0.1$ & $5e^{-4}$ & $128$ & $200$\\
        \midrule
        CIFAR100 & ResNet18 & \shortstack[c]{RandomCrop\\ RandomHorizontalFLip} & \shortstack[c]{SGD \\ momentum$=0.9$} & $0.1$ & $5e^{-4}$ & $128$ & $200$\\
        \midrule
        ImageNet & \shortstack[c]{Frozen dinov2\_vitb14\_reg\\ with two linear layers\\$2304\to512\to1000$} & \shortstack[c]{Resize\\ CenterCrop\\ TrivialAugmentWide} & \shortstack[c]{AdamW} & $0.001$ & $-$ & $256$ & $10$\\
        \midrule
        Waterbirds & \shortstack[c]{pretrained\\ efficientnet\_b3} & \shortstack[c]{Resize\\ RandomCrop\\RandomHorizontalFlip\\ TrivialAugmentWide} & \shortstack[c]{AdamW} & $0.00004$ & $5e^{-4}$ & $32$ & $300$\\
        \midrule
        CelebA & EfficientFormerV2 & \shortstack[c]{CenterCrop\\RandomHorizontalFlip\\ TrivialAugmentWide} & \shortstack[c]{AdamW} & $0.001$ & $5e^{-4}$ & $256$ & $5$\\
        \midrule
        iNaturalist & \shortstack[c]{pretrained\\ ResNet50} & \shortstack[c]{Resize\\ CenterCrop\\RandomHorizontalFlip\\ TrivialAugmentWide} & \shortstack[c]{AdamW} & $0.001$ & $5e^{-4}$ & $256$ & $5$\\
    \bottomrule
    \end{tabular}
\end{table}

\begin{table}[ht]
    \centering
    \caption{$\beta$ values used across all datasets and pruning rates.}
    \label{tab:reproduce_RCAP}
    \begin{tabular}{ccccc}
    \toprule
        Dataset & $50\%$ & $70\%$ & $80\%$ & $90\%$\\
    \toprule
        CIFAR10 & $3$ & $1$ & $2$ & $1$\\
        \midrule
        CIFAR100 & $3$ & $2$ & $2$ & $2$\\
        \midrule
        ImageNet & $\frac{1}{3}$ & $4$ & $2$ & $2$\\
        \midrule
        Waterbirds & $3$ & $\frac{1}{2}$ & $\frac{1}{2}$ & $\frac{1}{2}$\\
        \midrule
        CelebA & $2$ & $2$ & $3$ & $1$\\
        \midrule
        iNaturalist & $2$ & $3$ & $3$ & $\frac{1}{3}$\\
    \bottomrule
    \end{tabular}
\end{table}

\end{document}